\RequirePackage{fix-cm}
\documentclass[a4paper]{article}
\usepackage{graphicx}
\usepackage{subfigure} 
\usepackage{color}
\usepackage{amsmath}
\usepackage{amssymb}
\usepackage{algorithm}
\usepackage{algorithmic}
\usepackage{natbib}
\usepackage{wrapfig}
\usepackage{txfonts}

\DeclareMathAlphabet{\mathcal}{OMS}{cmsy}{m}{n}
\newcommand{\Min}[1]{\underset{#1}{\textrm{min}}}

\newcommand{\R}{\mathbb{R}}
\newcommand{\Argmin}[1]{\underset{#1}{\textrm{arg~min }}}
\newcommand{\stpro}{\mathcal{I}}
\DeclareMathOperator{\con}{con}
\DeclareMathOperator{\rk}{rk}
\DeclareMathOperator{\tr}{tr}
\DeclareMathOperator{\Gr}{Gr}
\DeclareMathOperator{\St}{St}
\DeclareMathOperator{\grad}{grad}
\newtheorem{proof}{Proof}

\newtheorem{lemma}{Lemma}

\begin{document}

\title{Robust PCA and subspace tracking from incomplete observations using $\ell_0$-surrogates\thanks{This work has been supported by the DFG excellence initiative research cluster \emph{CoTeSys}.}
}

\author{Clemens Hage and Martin Kleinsteuber\\ \\
Department of Electrical Engineering and Information Technology\\
Technische Universit\"{a}t M\"{u}nchen\\
Arcisstr.~21, 80333 Munich, Germany\\
\texttt{\{hage,kleinsteuber\}@tum.de}\\
\texttt{http://www.gol.ei.tum.de}
}

\maketitle

\begin{abstract}
Many applications in data analysis rely on the decomposition of a data matrix into a low-rank and a sparse component. Existing methods that tackle this task use the nuclear norm and $\ell_1$-cost functions as convex relaxations of the rank constraint and the sparsity measure, respectively, or employ thresholding techniques.
We propose a method that allows for reconstructing and tracking a subspace of upper-bounded dimension from incomplete and corrupted observations. It does not require any a priori information about the number of outliers. The core of our algorithm is an intrinsic Conjugate Gradient method on the set of orthogonal projection matrices, the so-called Grassmannian. Non-convex sparsity measures are used for outlier detection, which leads to improved performance in terms of robustly recovering and tracking the low-rank matrix.
In particular, our approach can cope with more outliers and with an underlying matrix of higher rank than other state-of-the-art methods.

\end{abstract}

\section{Introduction}
The detection of subspaces that best fit high-dimensional data is a challenging and important task in data analysis with countless applications. 
The classic Principal Component Analysis (PCA) is still the standard tool for searching the best rank-$k$ approximation  of a high-dimensional data set $X \in \mathbb{R}^{m \times n}$, where the approximation quality is measured in terms of the Frobenius norm.
On the one hand, this minimization problem can be solved easily via a Singular Value Decomposition (SVD), while on the other hand the choice of the Frobenius norm makes this approach highly vulnerable to heavy outliers. 
%
%
In the past years a lot of effort has been spent on methods that allow to find the best fitting subspace
despite the presence of heavy outliers in the measurements or missing data. 
%
%
Certain approaches in this area commonly known as \emph{Robust~PCA} stick closely to the classic PCA concept and robustify it by measuring the distance from the data points to the subspace with $\ell_1$-cost functions, cf.~\cite{Ding2006}, \cite{Kwak2008}.
Others, including the approach presented here, relate Robust PCA to the field of robust matrix completion.
That is, given some incomplete observations $\hat{X}$ of a corrupted and possibly noisy data matrix $X$, which is the superposition of a low-rank matrix $L$ and a sparse matrix $S$, the task is to recover $L$ and $S$.

\subsection{Related Work}
One of the most prominent approaches to Robust PCA in the past years is proposed by \cite{Candes2011} and \cite{wright:09}. The authors formulate a convex optimization problem by relaxing the hard rank constraint to a nuclear norm minimization and analyze how well the $\ell_1$-relaxation approximates the $\ell_0$-norm in the low-rank and sparse decomposition task. Methods to solve the problem include Singular Value Thresholding (\emph{SVT}, \cite{Cai2008}) and the exact (\emph{EALM}) or inexact (\emph{IALM}) augmented Lagrangian multiplier method \citep{Lin2010}, which are fast and comparably reliable if the underlying assumptions, i.e. low-rank-property and sparsity of outliers, are valid. The problem of reconstructing the low-rank matrix $L$ in the case where entire columns in the measurements are corrupted is considered by \cite{Chen2011}. A method called \emph{SpaRCS} is proposed by \cite{Waters2011}, which recovers the low-rank and the sparse part of a matrix from compressive measurements.

In many applications it is reasonable to assume that no additive Gaussian noise is present in the model, as noise is negligible compared to the signal power, e.g.~when dealing with high-quality sensors in the field of Computer Vision. However, there also exist approaches such as \emph{GoDec} \citep{Zhou2011} that explicitly model additional Gaussian noise. The method uses thresholding techniques and extends the low-rank and sparse decomposition to the noisy case. \emph{SpaRCS} and \emph{GoDec} aim at recovering both the low-rank matrix $L$ and the outliers $S$ from noisy measurements and therefore require an upper bound for the cardinality of the sparse component.

In contrast to the often-performed rank-relaxation, there also exist methods which fix the dimension of the subspace approximation, such as the greedy algorithm \emph{GECO} \citep{Shalev-Shwartz2011}. This method reconstructs a dataset iteratively based on SVD while increasing the rank of the approximation in each step.
A general framework for optimizing over matrices of a fixed rank has also been proposed by \cite{Shalit2010}. One drawback here is that the manifold requires a fixed rank $k$ and optimal points that may have a rank strictly lower than $k$ are not in the feasible set.

An alternative and elegant way to control the rank in terms of an upper bound is optimization on the set of fixed-dimensional subspaces, the so-called Grassmannian. As pointed out by \cite{Meyer2011}, optimizing on the Grassmannian offers many advantages, such as limited memory usage and a lower number of parameters to optimize. Although the optimization problem becomes non-convex, reliable performance can be achieved in practice.
In the work of \cite{Keshavan2010}, spectral techniques are combined with a learning algorithm on the Grassmannian. \cite{Boumal:11} propose a Riemannian trust-region method on the Grassmannian for low-rank matrix completion, which, however, does not consider heavy outliers. The \emph{GROUSE} algorithm \citep{Balzano2010} furthermore demonstrates that optimization on the Grassmannian allows to estimate the underlying subspace incrementally. Instead of batch-processing a data set, the samples can be processed one at a time, which makes it possible to track a subspace that varies over time, even if it is incompletely observed. A small portion of the data is used to obtain an initial subspace estimate and with every new incoming data sample this estimate is modified to follow changes in the dominant subspace over time.
While the method of \cite{Balzano2010} operates with $\ell_2$-cost functions, its recent adaptation \emph{GRASTA} \citep{He2012} performs a gradient descent on the Grassmannian and aims at optimizing an $\ell_1$-cost function to mitigate the effects of heavy outliers in the subspace tracking stage.
The authors overcome the non-differentiability of the $\ell_1$-norm by formulating an augmented Lagrangian optimization problem at the cost of doubling the number of unknown parameters.
\subsection{Our Contribution}
Although the $\ell_1$-norm leads to favorably conditioned optimization problems it is well-known that penalizing with non-convex $\ell_0$-surrogates allows reconstruction even in the case when $\ell_1$-based methods fail, see e.g.~\cite{Chartrand2008}.
Therefore, we propose a framework that combines the advantages of Grassmannian optimization with non-convex sparsity measures. Our approach focuses primarily on reconstructing and tracking the underlying subspace and can operate on both fully and incompletely observed data sets. The algorithm performs a low-rank and sparse decomposition. However, it does not require any information about the cardinality or the support set of the sparse outliers, thus being different from \emph{SpaRCS} and \emph{GoDec}.\\
In contrast to \emph{GRASTA} \citep{He2012}, the method presented in this paper directly optimizes the cost function and thus operates with less than half the number of unknowns. Like all optimization methods on the Grassmannian our algorithm allows to upper-bound the dimension of the underlying subspace and easily extends to the problem of robustly tracking this subspace.

Experimental results confirm that the proposed method can cope with more outliers and with an underlying matrix of higher rank than other state-of-the-art methods, thus extending possible areas of application from the strict low-rank and sparse decomposition to the more general area of robust dimensionality reduction.
In the following section we present an alternating minimization scheme and relate our approach to dimensionality reduction via PCA. Subsequently, we carefully derive and explain a Conjugate Gradient (CG) type algorithm on the Grassmannian for solving the individual minimization tasks. We then extend the static method by a dynamic subspace tracking algorithm and finally, we evaluate the performance of the proposed method with various $\ell_0$-surrogates and compare our approach to other state-of-the-art methods.

\section{Problem statement}

Let $X \in \R^{m \times n}$ be the data matrix of which we select the partial entries $\hat{X} = \mathcal{A}(X)$ using a linear operator. We consider the data model  $X = L + S$ where $L$ is of rank $\rk(L) \leq k$ and $S$ is sparse. Our aim is to recover $L$ from the observations $\hat{X}$. A direct approach leads to the numerically infeasible optimization problem
\begin{align}\label{eq:prob_infeas}
\Min{\rk{L} \leq k} & \|\hat{X}-\mathcal{A}(L)\|_0.
\end{align}
Since $\rk(L) \leq k$, we can write the matrix $L$ as the product $ L=U Y $, where $Y \in \R^{k \times n}$ and $U$ is an element of the so-called \emph{Stiefel manifold} $\St_{k,m} = \{U \in \mathbb{R}^{m \times k} | U^\top U = I_k\}$ with $I_k$ denoting the $(k \times k)$-identity matrix. This factorization of $L$ has the following practical interpretation: $U$ is an orthonormal basis of the robustly estimated subspace of the first $k$ principal components, hereafter referred to as \emph{(robust) dominant subspace}. Note, that neither $U$ nor $Y$ are uniquely determined in this factorization. Let $O(k)=\{\theta \in \mathbb{R}^{k \times k} | \theta^\top \theta = I_k\}$ define the set of orthogonal matrices, then $U$ can always be adjusted by an orthogonal matrix $\theta \in O(k)$,  such that $UY=U \theta \theta^\top Y$ with $\theta^\top Y$ having uncorrelated rows, leading to the $k$ (robustly estimated) principle component scores of $X$.

Our concession to the numerical infeasibility of \eqref{eq:prob_infeas} is to employ an alternative, possibly non-convex sparsity measure $h$. Some concrete examples can be found in Section \ref{sec:penfuns}.
As a result, problem \eqref{eq:prob_infeas} 
relaxes to the minimization problem
\begin{align}\label{eq:prob_infeas2}
\Min{U \in \St_{k,m},Y \in  \R^{k \times n}} & h(\hat{X}-\mathcal{A}(U Y)).
\end{align}
Problem \eqref{eq:prob_infeas2} can be addressed in two different ways. Either the full data set $\hat{X}$ is processed at once, resulting in one estimate for $U$ and $Y$, or the data vectors $\hat{x}$ are processed one at a time. In the latter case, while for each new sample $\hat{x}$ one obtains a corresponding optimal coordinate set $y$, the subspace estimate $U$ should vary smoothly over time and fit both recent and current observations. We will refer to the former problem as \emph{subspace reconstruction} or \emph{Robust PCA} and to the latter as \emph{subspace tracking}.

\section{An alternating minimization framework using $\ell_0$ surrogates}
Directly tackling problem \eqref{eq:prob_infeas2} has two severe drawbacks. The first is the above mentioned ambiguity in the factorization $UY$ and the second is the fact that reasonable sparsity measures are not smooth and thus forbid the use of fast smooth optimization methods.
We overcome these two problems by proposing an alternating minimization framework that gets rid of the ambiguity by iterating on a more appropriate geometric setting and allows us to use smooth approximations of $h$, which combine the
advantage of having fast smooth optimization tools at hand together with the strong capability of non-convex sparsity measures in reconstruction tasks.

So let $h_\mu\colon \R^{m \times n} \to \R$ with $\mu \in \mathbb{R}^+$ be a smooth approximation of $h$ such that $h_\mu$ converges pointwise to $h$ as $\mu$ tends to zero. Monotonically shrinking the smoothing parameter $\mu$ between alternating minimization steps allows to combine the advantages of both smooth optimization and $\ell_0$-like sparsity measures. Schematically, we tackle problem \eqref{eq:prob_infeas2} by iterating the following two  steps until convergence.

\noindent \textbf{Step \#1}\quad Let $L^{(i)}=U^{(i)} Y^{(i)}$ be the $i$-th iterate of the minimization process. In the first step, the estimate of the dominant subspace is improved by solving
\begin{align}\label{eq:opt_U}
U^{(i+1)}=\Argmin {U \in \St_{k,m}} & h_\mu(\hat{X}- \mathcal{A}(U U^\top L^{(i)})).
\end{align}
We master the ambiguity problem by reformulating \eqref{eq:opt_U} as a minimization task on the set of symmetric rank-$k$ projectors, which possesses a manifold structure and is known as the \emph{Grassmannian}
\begin{align}
\Gr_{k,m} := \{P \in \mathbb{R}^{m \times m} | P = U U^\top, U \in \St_{k,m}\} .
\end{align}
This results in the optimization problem
\begin{align}\label{eq:opt_over_GR}
P^{(i+1)}=\Argmin {P \in \Gr_{k,m}} & h_\mu(\hat{X}- \mathcal{A}(P L^{(i)})),
\end{align}
with $P^{(i+1)}=U^{(i+1)}(U^{(i+1)})^\top$.\\

\noindent \textbf{Step \#2} \quad In the second step, the coordinates with respect to $U^{(i+1)}$ of the projected data points are adjusted by solving
\begin{align}\label{eq:opt_over_R}
Y^{(i+1)}=\Argmin {Y \in \R^{k \times n}} & h_\mu(\hat{X}- \mathcal{A}(U^{(i+1)} Y)).
\end{align}
Then $\mu$ is decreased previous to the next iteration.
Since ordinary PCA of the data allows a reasonably good initialization of $U$ and $Y$, we initialize our algorithm with a truncated SVD of some matrix $X_0$ that is in accordance with the measurements, i.e. $\mathcal{A}(X_0)=\hat{X}$. The alternating scheme is summarized in Algorithm \ref{alg:noisefree}.
\begin{algorithm}
   \caption{Alternating scheme for Robust PCA}
\label{alg:noisefree}
\begin{algorithmic}
   \STATE {\bfseries Initialize:}
   \STATE Choose $X_0$, s.t.~$\mathcal{A}(X_0) = \hat{X}$.
   \STATE Obtain $U^{(0)}$ from $k$ left singular values of $X_0$.
   \STATE $Y^{(0)} = U^{(0)\,\top} X_0$, $L^{(0)} = U^{(0)}Y^{(0)}$, $P^{(0)} = U^{(0)}U^{(0)\,\top}$
   \STATE Choose $\mu^{(0)}$ and $\mu^{(I)}$, compute $c_\mu = \left(\tfrac{\mu^{(I)}}{\mu^{(0)}}\right)^{1/(I-1)}$
\FOR{$i = 1:I$}
\vspace{6pt}
   \STATE $P^{(i+1)} = \Argmin {P \in \Gr_{k,m}} h_{\mu^{(i)}}(\hat{X}- \mathcal{A}(P L^{(i)})) \quad$ \textbf{Step \#1}
   \STATE find $U^{(i+1)}$ \; s.t. \; $U^{(i+1)}U^{(i+1)\,\top} = P^{(i+1)}$
\vspace{6pt}
   \STATE $Y^{(i+1)} = \Argmin {Y \in \R^{k \times n}} h_{\mu^{(i)}}(\hat{X}- \mathcal{A}(U^{(i+1)} Y)) \quad$ \textbf{Step \#2}
\vspace{6pt}   
   \STATE $L^{(i+1)} = U^{(i+1)}Y^{(i+1)}$
   \STATE $\mu^{(i+1)} = c_\mu \mu^{(i)}$
\ENDFOR
\STATE $\hat{L} = U^{(I)}Y^{(I)}, \quad \hat{S} = X - \hat{L}$
\end{algorithmic}
\end{algorithm}
Only a finite number of alternating minimization steps is performed, each of which can be interpreted as an estimation for an appropriate initialization of the subsequent one. Its convergence analysis thus reduces to the question of convergence in the last iteration, which depends on the particular minimization algorithm and is briefly discussed in Section~\ref{subsec:cg_on_grass}.

\section{Optimizing the smooth sparsity measure}
The proposed alternating minimization scheme involves the two non-convex but \linebreak smooth optimization problems \eqref{eq:opt_over_GR} and \eqref{eq:opt_over_R}, both of which we minimize using a Conjugate Gradient type method due to its scalability and fast local convergence properties.  
The two proposed optimization methods are conceptually very similar but differ in the domains they operate on, as becomes clear from the cost functions and their respective gradients.
\begin{align}
\label{eq:costfct_P} &f_1 \colon \Gr_{k,m} \to \R, &P \mapsto h_\mu(\hat{X} - \mathcal{A}(PL)), \quad
&\nabla f_1 = -\mathcal{A}^\ast\Big(\nabla h_\mu(\hat{X} - \mathcal{A}(PL))\Big)\; L^\top\\
\label{eq:costfct_Y} &f_2 \colon \R^{k \times n} \to \R, &Y \mapsto h_\mu(\hat{X} - \mathcal{A}(UY)), \quad
&\nabla f_2 = -U^\top \mathcal{A}^\ast\Big(\nabla h_\mu(\hat{X} - \mathcal{A}(UY))\Big)
\end{align}
Note that $\mathcal{A}^\ast$ denotes the adjoint of the operator $\mathcal{A}$. CG methods for minimization problems in the Euclidean case, such as \eqref{eq:costfct_Y} are standard and well established. In contrast to this, the core of our algorithm, i.e.~minimizing \eqref{eq:costfct_P} is a geometric optimization problem on the real Grassmannian and thus requires additional concepts such as vector transport and retraction.

In the following, we will recall some general concepts and further collect the ingredients for our algorithm. In particular, we derive a new retraction that is crucial for the algorithm's computational performance. For a deeper and more general insight into the topic of Geometric Optimization we refer to the work of \cite{Absil2008}.

\subsection{Geometry of the Grassmannian}
Consider a projector $P \in \Gr_{k,m}$ as a point on the Grassmannian and let $\mathfrak{u}(m) := \{ \Omega \in \mathbb{R}^{m \times m} | \Omega^\top = -\Omega \}$ be the set of skew-symmetric matrices. Using the Lie bracket operator $\left[Z_1,Z_2\right] = Z_1 Z_2 - Z_2 Z_1$, the set of elements in the tangent space of $\Gr_{k,m}$ at $P$ is given by $T_P \Gr_{k,m} = \{[P,\Omega] \; | \; \Omega \in \mathfrak{u}_m\}$. In the following, we will endow $\R^{m \times m}$ with the Frobenius inner product
$\langle Z_1, Z_2 \rangle:= \tr(Z_1^\top Z_2)$ where $\tr(\cdot)$ denotes the trace operator, and consider the Riemannian metric on $\Gr_{k,m}$ accordingly as the restriction of this product to the respective tangent space.
The orthogonal projection of an arbitrary point $Z \in \R^{m \times m}$ onto the tangent space at $P$ is
\begin{align}
\label{eq:projectanspace}
\pi_P\colon \R^{m \times m} \to T_P \Gr_{k,m}, \quad Z \mapsto [P,[ P, Z_s]]
\end{align}
with $Z_s = \tfrac{1}{2}(Z + Z^\top)$ being the symmetric part of $Z$.\\

It is crucial for optimization procedures on manifolds to establish a relation between elements of the tangent space and corresponding points on the manifold, which motivates the usage of retractions. Conceptually, a retraction at some point $P$ is a mapping from the tangent space at $P$ to $\Gr_{k,m}$ with a local rigidity condition that preserves gradients at $P$.

The generic way of locally parameterizing a smooth manifold is via Riemannian exponential mappings. As they are costly to compute in general we perform an approximation based on the $QR$-decomposition. Let $Gl(m)$ be the set of invertible $(m \times m)$-matrices and let ${\mathcal R}(m) \subset Gl(m)$ be the set of
upper triangular matrices with positive entries on the diagonal.
It follows from the Gram-Schmidt orthogonalization procedure that the $QR$-decomposition
$O(m) \times {\mathcal R}(m) \to Gl(m), \;(Q,R) \mapsto QR$
is a diffeomorphism. Accordingly, every $A \in Gl(m)$ decomposes uniquely into $A=:A_Q A_R$ with $A_Q \in O(m)$ and $A_R \in {\mathcal R}(m)$. Moreover, it follows that the map
\begin{align}
q_\Omega \colon \R \to O(m), \quad q_\Omega(t):= (I_m + t \Omega)_Q
\end{align}
is smooth for all $\Omega \in \mathfrak{u}(m)$ with derivative at $0$ being $\dot{q}_\Omega(0)= \Omega$,
cf.~\cite{Kleinsteuber2007}. Using 
\begin{align}
\alpha_P^\prime\colon T_P \Gr_{k,m} \to \Gr_{k,m}, \quad \alpha^\prime_P(\xi):= q_{[\xi,P]}(1)\; P\;  \big(q_{[\xi,P]}(1)\big)^\top
\end{align}
and exploiting the properties of the exponential mapping, one can develop the following result:
\begin{lemma} \label{lemma:retract}
Consider arbitrary orthogonal matrices $\theta \in O(m)$. The mapping
\begin{align}\label{eq:retract}
\alpha_{P,\theta}\colon T_P \Gr_{k,m} \to \Gr_{k,m}, \quad \alpha_{P,\theta}(\xi):= \theta \; \big( q_{\theta^\top[\xi,P]\theta}(1) \big) \; \theta^\top P \; \theta \; \big( q_{\theta^\top[\xi,P]\theta}(1)\big)^\top \; \theta^\top
\end{align}
defines a set of retractions on $\Gr_{k,m}$.
\end{lemma}
\begin{proof}
The first condition, namely $\alpha_{P,\theta}(0)=P$, follows straightforwardly. It remains to show that $\tfrac{d}{dt}\big|_{t=0}{\alpha_{P,\theta}}(t \xi)=\xi$.\\
Firstly, note that $q_{\left[t\xi,P \right]}(1) = q_{\left[\xi,P \right]}(t)$. 
Since it has been verified that $\dot{q}_{\Omega}(0)= \Omega$ (cf.~\cite{Helmke07}) and $[\xi,P]$ is skew symmetric, it follows that
\begin{align*}
\frac{d}{dt}\big|_{t=0}{\alpha_{P,\theta}}(t \xi) &= \theta \dot{q}_{\theta^\top[\xi,P]\theta}(0)\theta^\top P \theta \big(q_{\theta^\top[\xi,P]\theta}(0)\big)^\top \theta^\top +
\theta q_{\theta^\top[\xi,P]\theta}(0) \theta^\top P \theta \big(\dot{q}_{\theta^\top[\xi,P]\theta}(0) \big)^\top \theta^\top\\
&=\theta \; \theta^\top [\xi,P] \theta \; \theta^\top P - P\theta \; \theta^\top [\xi,P] \theta \; \theta^\top\\
&=[P,[P,\xi]]\\
&= \xi
\end{align*}
To understand the last step of the equation, note that $\xi \in T_P \Gr_{k,m}$ and therefore, $\xi$ is invariant under the projection $\pi_P (\cdot).$
\hfill{$\Box$}
\end{proof}
As an associated vector transport, i.e. a mapping that for a given $\xi \in T_P \Gr_{k,m}$ transports the tangent element 
$\eta \in T_P \Gr_{k,m}$ along the retraction $\alpha_{P,\theta}(\xi)$ to the tangent space $T_{\alpha_{P,\theta}(\xi)} \Gr_{k,m}$, we choose
\begin{align}\label{eq:vectortrans}
\tau_{\xi,P,\theta} (\eta):= \theta \; \big( q_{\theta^\top[\xi,P]\theta}(1)\big) \; \theta^\top \eta \; \theta \; \big( q_{\theta^\top[\xi,P]\theta}(1)\big)^\top \theta^\top.
\end{align}
Note that in our algorithm the context of $\tau$ is always clear. Thus, we will drop the subscripts and simply write $\tau(\eta)$ for enhanced legibility.

\subsection{CG on the Grassmannian}\label{subsec:cg_on_grass}

In the following, we sketch how the well-known nonlinear CG method extends to the Grassmannian for minimizing a smooth function $f \colon \Gr_{k,m} \to \R$.
Recall that if $f$ is the restriction of a smooth function $\hat{f} \colon \R^{m \times m} \to \R$,
the Riemannian gradient in the tangent space is given by
\begin{align}
\label{eq:graddef}
\grad f (P) = \pi_P(\nabla \hat{f}),
\end{align}
where $\nabla \hat{f}$ is the common gradient of $\hat{f}$ in $\R^{m \times m}$.
The CG method on the Grassmannian can be outlined as follows. Starting at an initial point $P^{(0)} \in \Gr_{k,m}$, the Riemannian gradient $\Gamma^{(0)}$ can be computed and $\mathnormal{H}^{(0)} = -\Gamma^{(0)}$ is selected as initial search direction. In each iteration suitable step-size $t^{(i)}$ is determined using a backtracking line-search algorithm on the Grassmannian, cf. Algorithm \ref{alg:BTonGr}. The new iterate is then obtained via $P^{(i+1)} =  \alpha_{P^{(i)}}(t^{(i)} H^{(i)})$. Finally, the search direction
\begin{align}
\label{eq:dirupdateonGr}
H^{(i+1)} = -\Gamma^{(i+1)} + \beta^{(i)} \tau(H^{(i)})
\end{align}
is updated, where we consider two update rules for $ \beta^{(i)}$, namely
\begin{align}\label{eq:heststiefonGr}
\beta_{FR}^{(i)} = \frac{\langle \Gamma^{(i+1)},\Gamma^{(i+1)}\rangle}{\langle \Gamma^{(i)}, \Gamma^{(i)}\rangle}, \quad
\beta_{HS}^{(i)} = \frac{\langle \Gamma^{(i+1)},(\Gamma^{(i+1)} - \tau (\Gamma^{(i)}))\rangle}{\langle \tau (H^{(i)}),(\Gamma^{(i+1)} - \tau (\Gamma^{(i)}))\rangle}.
\end{align}
The former is a Riemannian adaption of the well-known Fletcher-Reeves update formula and guarantees convergence of our algorithm (see the subsequent remark). The latter is an adaptation of the Hestenes-Stiefel formula and typically leads to better convergence behavior in practice, which is why we use it for all our algorithms.

\begin{algorithm}[H]
\caption{Backtracking line search on Grassmannian}
\label{alg:BTonGr}
\begin{algorithmic}
	\STATE Choose $t_{init} > 0; c,\rho \in (0,1)$ and set $t \gets t_{init}$
	\REPEAT 
	\STATE $t \gets \rho t$
	\UNTIL $f \left( \alpha_P(t H^{(i)})\right) \leq f(P)+ c\, t \, \tr(\Gamma^{(i)\,\top} H^{(i)})$
	\STATE Choose step-size $t^{(i)}:=t$
\end{algorithmic}
\end{algorithm}
\noindent \textit{Remark} Convergence of the geometric CG method, which uses the retraction and vector transport as described above together with the Fletcher-Reeves update $\beta_{FR}$ and the strong Wolfe-Powell condition, is guaranteed by a result of \cite{Ring2012} in the sense that $\liminf_{i\to\infty}  \| {\Gamma}^{(i)} \| = 0$. 
Although we do not explicitly examine the strong Wolfe-Powell condition in Alg.~\ref{alg:BTonGr} due to the increased computational costs, convergence
 behavior is observed in practice.
%
\subsection{Implementation of CG on Grassmannian}
So far, the algorithm outlined above requires full $(m \times m)$-matrices for the iterates $P^{(i)}, \Gamma^{(i)}$ and $H^{(i)}$. This is a drastic limitation on the performance of a practical implementation. In this section we derive a new retraction and show how it can be used to avoid full matrix multiplication and to reduce the storage requirements tremendously. The key idea is to decompose the projection matrices $P \in \Gr_{k,m}$ into $P=UU^\top$ and to iterate on Stiefel matrices $U \in \St_{k,m}$ instead. Moreover, one can exploit the structure of the tangent space $T_\stpro \Gr_{k,m}$ at the standard projector $\stpro$, that is
\begin{align}
\stpro = \begin{bmatrix} I_k & 0 \\ 0 & 0 \end{bmatrix}, \quad 
T_\stpro \Gr_{k,m}= \left\{ \begin{bmatrix} 0 & A^\top \\ A & 0 \end{bmatrix} ~\Big |~ A \in \R^{(m-k) \times k} \right\}.
\end{align}
Given $U$, a large $QR$-decomposition of $U$ yields a fast way of constructing
\begin{align}\label{eq:Vdef}
V:=\begin{bmatrix}U & \vrule & U^\perp \end{bmatrix} \in O(m),
\end{align}
where $U^\perp \in \St_{(m-k),m}$ denotes a basis of the orthogonal complement of the subspace spanned by $U$.
Then, if $P = U U^\top$ and $\xi \in T_P \Gr_{k,m}$, the identities
\begin{align}
\label{eq:stpro}
V^\top P V = \stpro  \quad \text{and} \quad V^\top \xi V = \begin{bmatrix} 0 & A^\top \\ A & 0 \end{bmatrix}
\end{align}
hold for some $A \in \R^{(m-k) \times k}$. Therefore, instead of storing the full $P$ and $\xi$ it is sufficient to store $U$ and $A$. Formally, this defines the bijection
\begin{align}
\label{eq:conV}
\con_V \colon T_P \Gr_{k,m} \to \R^{(m-k) \times k}, \; \con_V(\xi)= A.
\end{align}
The projection onto $T_P \Gr_{k,m}$ follows from a straightforward calculation.

\begin{lemma}
Let $V=\begin{bmatrix}U & \vrule & U^\perp \end{bmatrix}$,$P=UU^\top$ and $Z_s = \tfrac{1}{2}(Z + Z^\top)$ be defined as above and the orthogonal projection onto the tangent space at $P$ be denoted by $\pi_P$. Then the identity
\begin{align}
\con_V\left(\pi_P (Z)\right) = (U^\perp)^\top Z_s U
\end{align}
holds.
\end{lemma}
\begin{proof}
Using the definition of $\pi_P (Z)$ and the fact that $V^\top P V = \mathcal{I}$,
\begin{align*}
V^\top \pi_P (Z) V &= [V^\top P V\;, V^\top[P,Z_s]V]\\
&=[\mathcal{I},[\mathcal{I},V^\top Z_s V]]
\end{align*} 
and the same structure as in \eqref{eq:stpro} is obtained.
\hfill{$\Box$}
\end{proof}
Lemma \ref{lemma:retract} allows to choose a particular retraction that is easy to compute. Consider
$A=\con_V (\xi)$ and let 
\begin{align}\label{eq:QRofA}
A = \theta_A \begin{bmatrix} R\\  0 \end{bmatrix}
\end{align}
 be the large $QR$-decomposition of $A$, with $\theta_A \in O(m-k)$ and $R$ an upper triangular (not necessarily invertible) $(k \times k)$-matrix.
Furthermore, define
\begin{align}
\label{eq:Mmatrix}
M(R) :=
\begin{bmatrix}
I_k & -R^\top \\ R & I_k
\end{bmatrix}
\end{align}
and its $Q$-factor
$\theta_M \in O(2k).$
\begin{lemma}\label{lemma:retract_IMPL}
Let $\alpha_{P,\theta}(\xi)$ be a retraction as in \eqref{eq:retract}, where $\theta$ is chosen as
$\theta:=V \begin{bmatrix} I_k & 0 \\ 0 & \theta_A \end{bmatrix}.$
Then the Stiefel matrix
\begin{align}
\label{eq:stiefelupdate}
\widetilde{U}:= \theta 
\begin{bmatrix} \theta_M & 0\\ 0 & I_{m-2k} \end{bmatrix}
\begin{bmatrix} I_k \\ 0 \end{bmatrix} \in \St_{k,m} 
\end{align}
satisfies $\alpha_{P,\theta}(\xi)=\widetilde{U} \widetilde{U}^\top$.
\end{lemma}
\begin{proof}
One can deduce
$q_{\theta^\top[\xi,P]\theta}(1) = (I_m + \theta^\top[\xi,P]\theta)_Q =\begin{bmatrix} \theta_M & 0 \\ 0 & I_{m-2k}  \end{bmatrix}$ from

\begin{align*}
\theta^\top [\xi,P] \theta &=
\begin{bmatrix} I_k & \\ & \theta_A^\top \end{bmatrix} [V^\top \xi V, V^\top P V]
\begin{bmatrix} I_k & 0\\ 0 & \theta_A \end{bmatrix}\\
&= 
\begin{bmatrix} I_k & 0\\ 0 & \theta_A^\top \end{bmatrix}
\begin{bmatrix} 0 & -A^\top\\ A & 0 \end{bmatrix}
\begin{bmatrix} I_k & 0 \\ 0 & \theta_A \end{bmatrix}\\
&=
\begin{bmatrix} 0 & -A^\top \theta_A\\ \theta_A^\top A & 0 \end{bmatrix}\\
&=
\begin{bmatrix} 0 &
\begin{bmatrix}-R^T & 0 \end{bmatrix}\\
\begin{bmatrix}R\\0 \end{bmatrix} & 0 \end{bmatrix} .
\end{align*}
Therefore, the retraction is
\begin{align*}
\alpha_{P,\theta}(\xi) &= \theta 
\begin{bmatrix} \theta_M & 0 \\ 0 & I_{m-2k}  \end{bmatrix} \theta^\top P \theta 
\begin{bmatrix} \theta_M^\top & 0 \\ 0 & I_{m-2k}  \end{bmatrix} \theta^\top\\
&=V
\begin{bmatrix} I_k & 0 \\ 0 & \theta_A \end{bmatrix} 
\begin{bmatrix} \theta_M & 0 \\ 0 & I_{m-2k} \end{bmatrix}  \underbrace{ 
\begin{bmatrix} I_k & 0 \\ 0 & \theta_A^\top \end{bmatrix} V^\top P V 
\begin{bmatrix} I_k & 0 \\ 0 & \theta_A \end{bmatrix}}_{\mathcal{I}}  
\begin{bmatrix} \theta_M^\top & 0 \\ 0 & I_{m-2k} \end{bmatrix}
\begin{bmatrix} I_k & 0 \\ 0 & \theta_A^\top \end{bmatrix} V^\top\\
&=V
\begin{bmatrix} I_k & 0 \\ 0 & \theta_A \end{bmatrix}
\begin{bmatrix} \theta_M & 0 \\ 0 & I_{m-2k} \end{bmatrix}
\begin{bmatrix} I_k\\ 0\end{bmatrix} \; \begin{bmatrix} I_k & 0\end{bmatrix} 
\begin{bmatrix} \theta_M^\top & 0 \\ 0 & I_{m-2k} \end{bmatrix}
\begin{bmatrix} I_k & 0 \\ 0 & \theta_A^\top \end{bmatrix} V^\top,
\end{align*}
which completes the proof.
\hfill{$\Box$}
\end{proof}

\noindent Similarly to the retraction, the vector transport from \eqref{eq:vectortrans} is simplified, as is described in the following.
\begin{lemma}
Let $\theta, \xi$, $P$, $\widetilde{U}$ be as above and $\widetilde{V}:= \begin{bmatrix}\widetilde{U} & \vrule & \widetilde{U}^\perp \end{bmatrix}$. Then for $\eta \in T_P \Gr_{k,m}$ and $B:=\con_V (\eta)$, the identity
$
\con_{\widetilde{V}}(\tau_{\xi,P,\theta}(\eta))=\theta_A^\top B
$
holds.
\end{lemma}
\begin{proof}
From \eqref{eq:stpro} we have $V^\top \eta V = \begin{bmatrix} 0 & B^\top\\ B & 0\end{bmatrix}$. The vector transport can thus be written as
\begin{align*}
\tau_{\xi,P,\theta}(\eta) =\theta 
\begin{bmatrix} \theta_M & 0 \\ 0 & I_k \end{bmatrix}
\begin{bmatrix} I_k & 0 \\ 0 & \theta_A^\top \end{bmatrix} \; 
\begin{bmatrix} 0 & B^\top\\ B & 0\end{bmatrix} \; 
\begin{bmatrix} I_k & 0 \\ 0 & \theta_A \end{bmatrix}
\begin{bmatrix} \theta_M^\top & 0 \\ 0 & I_k \end{bmatrix} \theta^\top.
\end{align*}
Using the fact that
\begin{align*}
\theta^\top 
\begin{bmatrix}\widetilde{U} & \vrule & \widetilde{U}^\perp \end{bmatrix} &= \theta^\top
\begin{bmatrix}\theta
\begin{bmatrix} \theta_M & 0 \\ 0 & I_{m-2k} \end{bmatrix}
\begin{bmatrix} I_k\\ 0 \end{bmatrix} \; & \vrule & \; \theta
\begin{bmatrix} \theta_M & 0 \\ 0 & I_{m-2k} \end{bmatrix}
\begin{bmatrix} 0\\ I_{m-k} \end{bmatrix}
\end{bmatrix}
=
\begin{bmatrix} \theta_M & 0 \\ 0 & I_{m-2k} \end{bmatrix}
\end{align*}
we can show that
\begin{align*}
\widetilde{V}^\top \tau_{\xi,P,\theta}(\eta) \widetilde{V} &=
\begin{bmatrix}\widetilde{U} & \vrule & \widetilde{U}^\perp \end{bmatrix}^\top \theta 
\begin{bmatrix} \theta_M & 0 \\ 0 & I_{m-2k} \end{bmatrix}
\begin{bmatrix} I_k & 0 \\ 0 & \theta_A^\top \end{bmatrix} 
\begin{bmatrix} 0 & B^\top\\ B & 0\end{bmatrix}
\begin{bmatrix} I_k & 0 \\ 0 & \theta_A \end{bmatrix}
\begin{bmatrix} \theta_M^\top & 0 \\ 0 & I_{m-2k} \end{bmatrix} \theta^\top
\begin{bmatrix}\widetilde{U} & \vrule & \widetilde{U}^\perp\end{bmatrix}\\
&= 
\begin{bmatrix} I_k & 0 \\ 0 & \theta_A^\top \end{bmatrix} 
\begin{bmatrix} 0 & B^\top\\ B & 0\end{bmatrix}
\begin{bmatrix} I_k & 0 \\ 0 & \theta_A \end{bmatrix}\\
&=
\begin{bmatrix} 0 & B^\top \theta_A\\ \theta_A^\top B & 0\end{bmatrix}
\end{align*}
and $\con_{\widetilde{V}}(\tau_{\xi,P,\theta}(\eta))$ follows again from the matrix structure.
\hfill{$\Box$}
\end{proof}

Algorithm \ref{alg:stiefel_cg} illustrates how the minimization of \eqref{eq:prob_infeas2} on the Grassmannian is efficiently implemented in practice. Note that therein, $H$ and $G$ denote the preimages of the respective bijection $\con_V$ and thus are of dimension $(m-k) \times k$.
We do not further discuss the CG method in the Euclidean space as it is a standard method. In all minimization procedures, the convergence is evaluated by observing the progress in decreasing the cost function.
\begin{algorithm}[H]
\caption{Implementation of CG on $Gr_{k,m}$}
\label{alg:stiefel_cg}
 \begin{algorithmic}
\INPUT $U = U^{(i)}$
\vspace{6pt}
\STATE Obtain $V$ and $U^{\perp}$ from Eq. \eqref{eq:Vdef}
\STATE Compute $G = (U^\perp)^\top \nabla \tilde{f}(UU^\top) U$
\STATE $H = -G$
\REPEAT
   \STATE Obtain $\theta_H,R$ from $H$ as in Eq. \eqref{eq:QRofA}
   \STATE Determine step-size $t$ acc. Algorithm \ref{alg:BTonGr}
   \STATE Obtain $\theta_M$ from $QR$-dec.~of $M(t R)$ \eqref{eq:Mmatrix}
   \STATE Update $U$ according to \eqref{eq:stiefelupdate}
   \STATE Update $V, U^\perp,G$ as above
   \STATE Compute $\tau (G) = \theta_H^\top G$ and $\tau (H) = \theta_H^\top H$
   \STATE Update $H$ following \eqref{eq:dirupdateonGr} and \eqref{eq:heststiefonGr}
\UNTIL{converged}
\vspace{6pt}
\OUTPUT $U^{(i+1)}$
 \end{algorithmic}
\end{algorithm}

\section{Robust subspace tracking}

In this section, we extend the aforementioned mathematical tools to robustly track the underlying subspace. To that end, we require our sparsity measure $h$ to be separable, meaning that the sparsity measure of a matrix $A=[a_1,a_2,\dots,a_n]$ consists of the sum of sparsity measures of its columns $a_i$. Note that all sparsity measures given in Eq.~\eqref{eq:l0surrogates} fulfill that condition. By slight abuse of notation, we write
\begin{align}
h\Big(\begin{bmatrix}A & \vrule & a \end{bmatrix}\Big)=h(A)+h(a),
\end{align}
where $a \in \R^m$ is the last column of the matrix $\begin{bmatrix}A & \vrule & a \end{bmatrix}$. Assume now that the current observation matrix $\hat{X}^{(i)}$ is updated by a new observation vector $\hat{x}^{(i+1)}$, leading to a new observation matrix $\hat{X}^{(i+1)} = \begin{bmatrix}\hat{X} & \vrule & \hat{x} \end{bmatrix}$. Following  \eqref{eq:prob_infeas2}, the new optimization problem is 
\begin{align}\label{eq:prob_infeas2_tracking}
\Min{U \in \St_{k,m},Y \in  \R^{k \times n},y \in \R^k} & h\Big(\begin{bmatrix}\hat{X} & \vrule & \hat{x} \end{bmatrix}-\mathcal{A}\Big(U \begin{bmatrix}Y & \vrule & y \end{bmatrix}\Big)\Big).
\end{align}
We show how the above problem can profit from knowledge of the optimal subspace in the previous iteration. Due to the separability of $h$ and the properties of $\mathcal{A}$ we can separate the current observation from the previous optimization problem by rewriting
\begin{align}
h\Big(
\begin{bmatrix}\hat{X} & \vrule & \hat{x} \end{bmatrix}-\mathcal{A}\Big(U 
\begin{bmatrix}Y & \vrule & y \end{bmatrix}\Big)\Big) = h(\hat{X}-\mathcal{A}(U Y)) + h(\hat{x}-\mathcal{A}(U y)).
\end{align}
Furthermore, we introduce a weighting factor $0 < w < 1$ for incoming observations, which will play the role of a \emph{forgetting factor} in the actual tracking algorithm.
The optimization problem for subspace tracking now reads as
\begin{align}\label{eq:prob_infeas2_tracking2}
\Min{U \in \St_{k,m},Y \in  \R^{k \times n},y \in \R^k} & (1-w)h(\hat{X}-\mathcal{A}(U Y)) + w h(\hat{x}-\mathcal{A}(U y)).
\end{align}
Assume $P^{(i)}$ and $L^{(i)} = U^{(i)} Y^{(i)}$ that minimize $h(\hat{X}-\mathcal{A}(L))$ and an initial estimate $l_0 = P^{(i)}x_0$ with $\mathcal{A}(x_0) = \hat{x}^{(i+1)}$ are available for the current data vector. Then a two-step update rule similar to \eqref{eq:opt_over_GR} and \eqref{eq:opt_over_R} can be formulated. In a first step, update the dominant subspace estimate
\begin{align}
\label{eq:tracking_U}
P^{(i+1)} \approx \Argmin{P \in \Gr_{k,m}}{(1-w)h(\hat{X}-\mathcal{A}(P L^{(i)})) + w h(\hat{x}-\mathcal{A}(P  l_0))}
\end{align}
via gradient descent on $Gr_{k,m}$ and decompose $P^{(i+1)}=U^{(i+1)}U^{(i+1) \top}$. Note that the gradient $\Gamma^{(i+1)}$ can be updated from the previous $\Gamma^{(i)}$ and the current observation. Then adjust the coordinates of the current low-rank estimate by optimizing
\begin{align}
\label{eq:tracking_y}
y^{(i+1)}= \Argmin{y\in \R^k}{(1-w)h(\hat{x}-\mathcal{A}(U^{(i+1)} y))}.
\end{align}

While \eqref{eq:tracking_y} can be solved in the same fashion as before by employing a CG method in Euclidean space, the solution of problem \eqref{eq:tracking_U} is approximated by one gradient descent step on the Grassmannian, i.e. 
\begin{align}
\label{eq:tracking_gradstep}
P^{(i+1)}=\alpha_{P^{(i)}}\Big(\beta \Gamma^{(i+1)}\Big) = \alpha_{P^{(i)}}\Big(\beta ((1-w)\Gamma^{(i)} + w \gamma^{(i+1)})\Big).
\end{align}
In this formula, $\beta$ is a suitable step size from standard gradient descent methods,
\begin{align}
\label{eq:defGamma}
\Gamma^{(i)}= grad_P h(\hat{X}-\mathcal{A}(P L^{(i)}))
\end{align}
is the Riemannian gradient at the preceding iteration, and
\begin{align}
\label{eq:defgamma}
\gamma^{(i+1)}= grad_P h(\hat{x}-\mathcal{A}(P l^{(i+1)}))
\end{align}
the Riemannian gradient for the current observation.
The tracking scheme is illustrated in Alg. \ref{alg:tracking}.
\begin{algorithm}[H]
   \caption{Subspace tracking}
\label{alg:tracking}
\begin{algorithmic}
   \STATE {\bfseries Initialize:}
   \STATE Select initial data set $\hat{X}_0$, consisting of at least $k$ data vectors.
   \STATE Estimate $U^{(0)}$ and $Y^{(0)}$ following the alternating minimization outlined in Alg. \ref{alg:noisefree}
   \STATE $L^{(0)} = U^{(0)}Y^{(0)}$
   \STATE $P^{(0)} = U^{(0)}U^{(0)\,\top}$
   \STATE Compute $\Gamma^{(0)} (\hat{X}_0,P^{(0)},L^{(0)})$ according to \eqref{eq:defGamma}
\FOR{each new sample $\hat{x} = \hat{x}^{(i+1)}$}
   \STATE Choose $w$ and $\mu$ as well as $x_0$, s.t.~$\mathcal{A}(x_0) = \hat{x}$.
   \STATE Initialize $y_0 = U^{(i)\; \top}x_0$ and $l_0= U^{(i)}y_0$
   \STATE Compute $\gamma^{(i+1)}(x_0,P^{(i)},l_0)$  according to \eqref{eq:defgamma}
   \STATE Update $\Gamma^{(i+1)}=(1-w)\Gamma^{(i)} + w \gamma^{(i+1)}$
   \STATE Update $P^{(i+1)} = \alpha_{P^{(i)}}\Big(\beta \Gamma^{(i+1)}\Big) $\quad cf.~\eqref{eq:tracking_U}, \eqref{eq:tracking_gradstep}
   \STATE find $U^{(i+1)}$ \; s.t. \; $U^{(i+1)}U^{(i+1)\,\top} = P^{(i+1)}$
   \STATE $y^{(i+1)}= \Argmin{y\in \R^k}{h_\mu(\hat{x}-\mathcal{A}(U^{(i+1)} y))}$ \quad cf.~\eqref{eq:tracking_y} 
   \STATE $l^{(i+1)} = U^{(i+1)}y^{(i+1)}$
\ENDFOR
\end{algorithmic}
\end{algorithm}
To further reduce computational cost, the following lemma proves useful.
\begin{lemma}
Define $V=\begin{bmatrix}U & \vrule & U^\perp \end{bmatrix}$ and the Riemannian Gradients $\Gamma$ and $\gamma$ as above and denote by $G = (U^\perp)^\top \Gamma U$ and $g = (U^\perp)^\top \gamma U$ their respective preimages according to \eqref{eq:conV}. Then the identity
\begin{align*}
\con_V ((1-w)\Gamma + w \gamma) = (1-w)G + wg
\end{align*}
holds.
\end{lemma}
The proof follows from the linearity of \eqref{eq:conV}.\\

Using this result, the computational effort reduces tremendously, because explicit computation of $\gamma^{(i+1)}$ and \mbox{$\alpha_{P^{(i)}}(\beta \Gamma^{(i+1)})$} is not required for the new iterate $U^{(i+1)}$. Considering \eqref{eq:costfct_P} for the case of a single data vector, $\gamma$ and with it $g$ become rank-one matrices. Thus, it is sufficient to compute $G$ and its $QR$-decomposition in the initialization phase and then to perform lightweight rank-one updates (cf. \cite{Golub96}) on the $Q$ and $R$ factors to update $G$ and $U$ in each iteration. The detailed procedure is illustrated in Algorithm \ref{alg:rankoneupdate}.
\begin{algorithm}[H]
\caption{Implementation of the Gradient descent update}
\label{alg:rankoneupdate}
 \begin{algorithmic}
\INPUT $x_0,l_0,P^{(i)},V=\begin{bmatrix}U^{(i)} & \vrule & U^{\perp\;(i)} \end{bmatrix}$, $\theta_G^{(i)}$ and $R^{(i)}$ from the $QR$-decomposition of $G^{(i)}$ as in \eqref{eq:QRofA}
\vspace{6pt}
\STATE Compute $\gamma(x_0,P^{(i)},l_0)$ and $g = \con_V(\gamma)$
\STATE Find $u$, $v$ s.t. $uv^\top = g$
\STATE Obtain $\theta_G^{(i+1)}$ and $R^{(i+1)}$ via rank-one-update from $\theta_G^{(i)}, R^{(i)}, u$ and $v$
\STATE Compute $\theta=V \begin{bmatrix} I_k & 0 \\ 0 & \theta_G^{(i+1)} \end{bmatrix}$
\STATE Determine optimum step-size $t$ from line-search along $H=-G^{(i+1)}$
\STATE Obtain $\theta_M$ from $QR$-decomposition of $M(tR^{(i+1)})$ \eqref{eq:Mmatrix}
\STATE Update U according to \eqref{eq:stiefelupdate}
\vspace{6pt}
\OUTPUT $U^{(i+1)}, \theta_G^{(i+1)}, R^{(i+1)}$
 \end{algorithmic}
\end{algorithm}

\section{Experiments and Evaluation}
This section gives an overview of the actual implementation and the performance of the presented algorithms. Firstly, we refer to practical issues such as the selection of suitable penalty functions and the choice of parameters. In the following, we evaluate the performance of our approach in the Robust PCA task and demonstrate how the proposed method outperforms other state-of-the art algorithms. Concluding the experiments, we illustrate the method's behavior on real-world data by tracking a low-dimensional subspace in a visual background reconstruction task.

\subsection{Smooth penalty functions}
\label{sec:penfuns}
Inspired by the work of \cite{Gasso2009}, we investigate the following smoothed $\ell_0$-surrogates:
\begin{subequations}
\label{eq:l0surrogates}
\begin{align}
\label{eq:lpnorm}
h^\textrm{lp}_{\mu} \colon \R^{m \times n} &\to \R^+, \quad X \mapsto \sum_{j=1}^n\sum_{i=1}^m \left(x_{ij}^2 + \mu\right)^\frac{p}{2}, \quad 0 < p < 1 \quad \text{(lpnorm)}\\
\label{eq:lognorm}
h^\textrm{log}_{\mu} \colon \R^{m \times n} &\to \R^+, \quad X \mapsto \sum_{j=1}^n\sum_{i=1}^m \mbox{log}\left(1 + \tfrac{x_{ij}^2}{\mu} \right) \quad \text{(logarithm)}\\
\label{eq:atansquare}
h^\textrm{atan}_{\mu} \colon \R^{m \times n} &\to \R^+, \quad X \mapsto \sum_{j=1}^n\sum_{i=1}^m \mbox{atan}^2\left(\tfrac{x_{ij}}{\mu}\right) \quad \text{(atan)}
\end{align}
\end{subequations}
Using this kind of cost function instead of e.g.~the $\ell_1$-norm has several advantages. The smoothing parameter $\mu$ can be tuned in a way that the cost function either penalizes large outliers (similar to the Frobenius norm) or enforces sparsity on the outliers regardless of their magnitude. Within a practical implementation it is observed that larger values of $\mu$ lead to faster convergence of the optimization algorithm while small values, as expected, lead to much sparser residuals. We profit from this flexibility by adjusting the smoothing parameter and thus the resulting sparsity. Specifically, in the subspace reconstruction task we prefer faster convergence in the beginning in order to quickly obtain a reliable rough estimate of the subspace and reduce $\mu$ after each alternation. In the limit $\mu \rightarrow 0$ the cost functions behave similarly as the $\ell_0$-norm, which leads to the best results, as the following experiments demonstrate.

\subsection{Numerical experiments for the subspace reconstruction task}
In order to have a quantitative measure for the recovery performance of the robust PCA algorithms, we conducted numerical experiments with artificial test data and employed the available ground truth for a quantitative evaluation.
We compare our algorithm to \emph{SpaRCS} \citep{Waters2011}, \emph{RPCA} (\emph{EALM} and \emph{IALM}) \citep{Lin2010}, \emph{GRASTA} \citep{He2012} and \emph{GoDec} \citep{Zhou2011}, and use MATLAB implementations provided by the authors.
In all experiments we constructed the test data $X \in \mathbb{R}^{m \times n}$ to be the sum of a matrix $L$ of fixed rank $k$ (with $k < m$) and a sparse matrix $S$. The low-rank component is obtained by computing the singular value decomposition $U \Sigma V^\top$ of a zero-mean, $\mathcal{N}(0,1)$-distributed random matrix, assigning zero to all singular values $\sigma_i, i = k+1 \dots m$ and reconstructing the matrix $F = U\tilde{\Sigma}V^\top$.
In order to control the magnitude of the entries for varying $k$ we scale the entries of $F$ to a unit sample standard deviation, obtaining a normalized $L = \tfrac{1}{\text{std}(F)} F$. The entries of $S$ are randomly placed and uniformly distributed in $[-5, 5]$, which makes a well-proportioned relation between data points and outliers.
Some other comparisons follow the data model of \cite{Candes2011}, where the low-rank matrix is the product of two $\mathcal{N}(0,\tfrac{1}{n})$-distributed random matrices. The entries of $S$ follow a Bernoulli distribution and thus surpass the amplitude of the data points by several orders. In our evaluation the outliers cannot be detected just from their magnitude, but they are still large enough to severely distort the estimated subspace.

\begin{figure}[ht]
\begin{center}
\subfigure[lpnorm]{\includegraphics[width=0.22\textwidth]{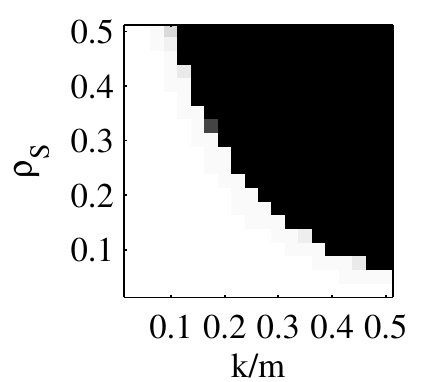}}
\subfigure[logarithm]{\includegraphics[width=0.22\textwidth]{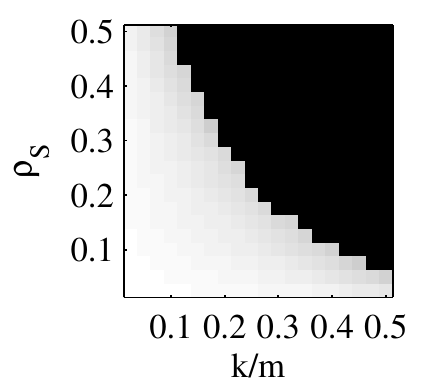}}
\subfigure[atan]{\includegraphics[width=0.22\textwidth]{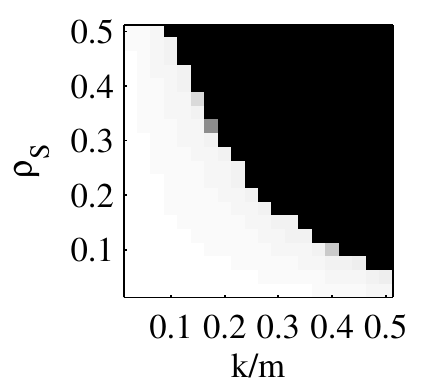}}
\subfigure[SpaRCS]{\includegraphics[width=0.22\textwidth]{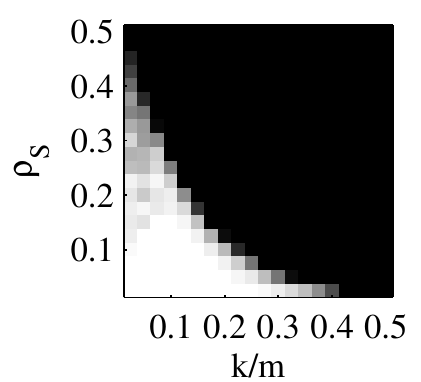}}\\
\subfigure[RPCA (EALM)]{\includegraphics[width=0.22\textwidth]{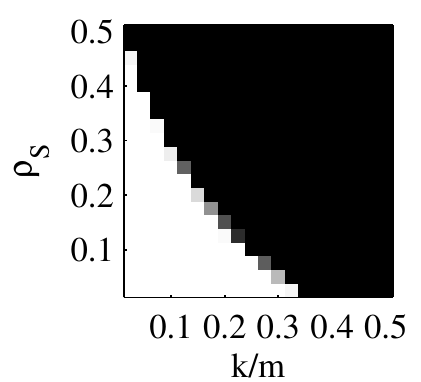}}
\subfigure[RPCA (IALM)]{\includegraphics[width=0.22\textwidth]{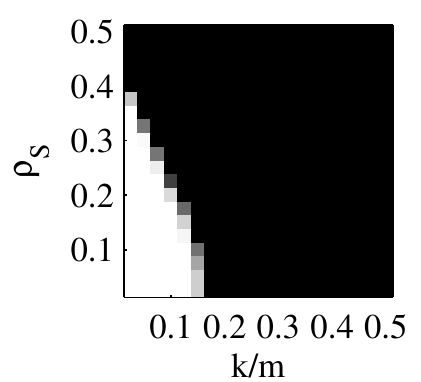}}
\subfigure[GRASTA]{\includegraphics[width=0.22\textwidth]{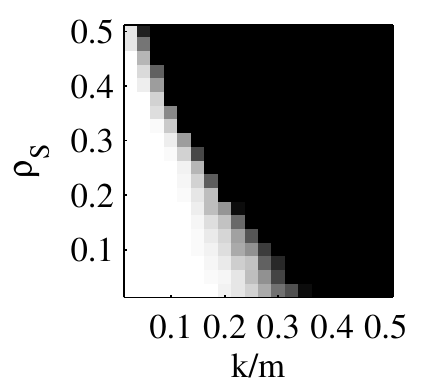}}
\subfigure[GoDec]{\includegraphics[width=0.22\textwidth]{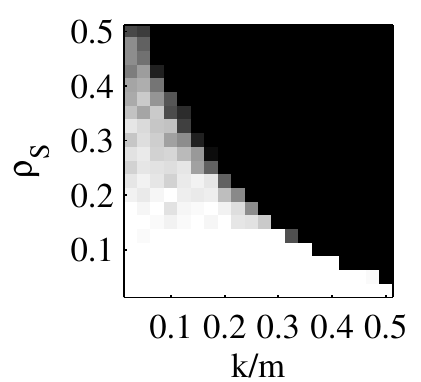}}
\caption{Subspace reconstruction performance measured by phase transitions in rank and sparsity (dark area: reconstruction failed). Our method (a)-(c) achieves the highest reconstruction rate.}
\label{fig:phase}
\end{center}
\end{figure}
A widely used performance measure for subspace reconstruction are the phase transitions in rank and sparsity. A data set of dimension $m = n = 400$ is generated and the subspace recovery is evaluated while varying both the relative rank $k/m$ and the relative sparsity $\rho_S=\|S\|_0/mn$ in the range of $0.025$ to $0.5$. The reconstruction is considered abortive if $\frac{\|L - \hat{L}\|_F}{\|L\|_F} > 0.05$.

The results in Figure \ref{fig:phase} demonstrate that the proposed framework (a)-(c) using $\ell_0$-surrogates covers a broad range of scenarios and surpasses \emph{SpaRCS}, \emph{GRASTA} and the \emph{RPCA} methods in the subspace reconstruction task. Especially in the cases of either $k/m$ being very small and $\rho_S$ being very large or vice versa, our algorithm is still able to recover the subspace while other methods fail. Out of all compared methods, the \emph{GoDec} algorithm comes closest to our performance - however, it has to be stressed that we must feed the exact cardinality of $S$ into the \emph{GoDec} algorithm, which is not available in a real-world application.
We chose the number of iterations between the alternating minimization steps as $I=50$ for all phase transition plots. We choose $p=0.5$ and shrink $\mu$ from $0.9$ to $10^{-4}$ for \eqref{eq:lpnorm}, from $2$ to $0.005$ for \eqref{eq:lognorm} and, respectively, from $2$ to $0.05$ for \eqref{eq:atansquare}.
For comparison we set the actual dimension of the original subspace  as an upper bound for the rank. However, the algorithm performs equally well if the upper bound on the rank is moderately higher. To show this, we investigate the case $\rho_S = 0.2$, $m=n=400$ and $k=80$ (i.e.~$\tfrac{k}{m}=0.2$) again for \eqref{eq:lpnorm} using the previous parameters but this time varying the upper bound on the rank. As the reconstruction accuracies for $S$ and $L$ in Figure \ref{fig:overest} illustrate, the separation quality is equally good if the subspace dimension is slightly overestimated. Thus, even if the true rank can only be assessed (e.g.~from the largest singular values), the proposed algorithm still performs reliably in a scenario where other methods fail.
\begin{figure}
\includegraphics[width=\textwidth]{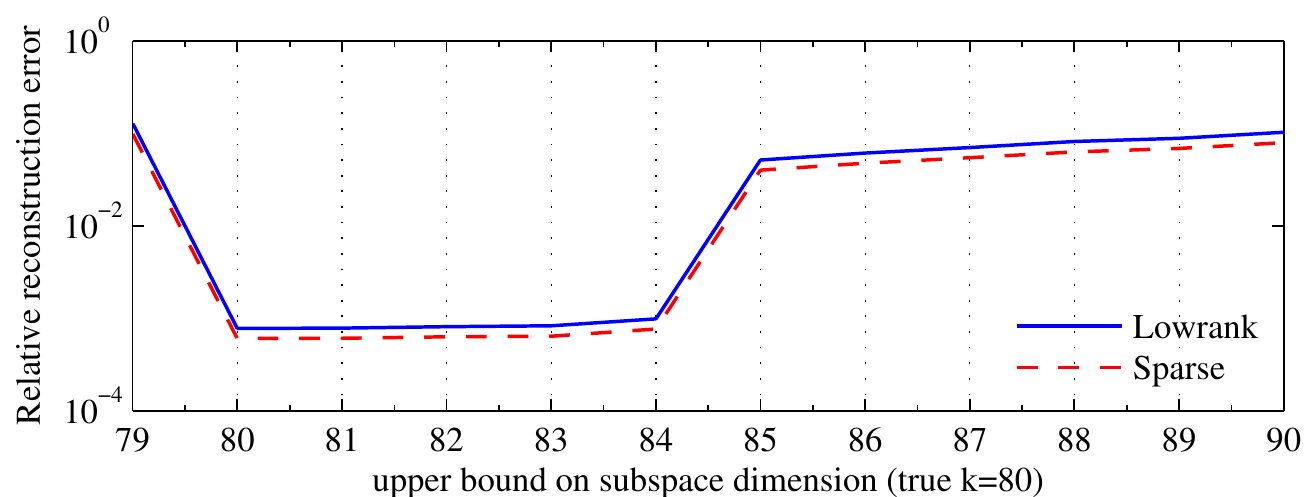}
\caption{Moderate overestimation of the subspace dimension does not alter the reconstruction accuracy}
\label{fig:overest}
\end{figure}

\subsection{Incomplete observations}
In a second series of experiments we evaluate our algorithm for the case of data being incompletely observed and compare the performance against \emph{SpaRCS}, which is especially designed for compressive measurements. The results in Figure \ref{fig:incomplete} illustrate that our method recovers a broader range of configurations than the competing \emph{SpaRCS}, which has its strengths especially for very low-rank matrices. In general, with less samples being available, the lower the dimension of the underlying subspace and the sparser the outliers have to be for a successful reconstruction. We stress that (i) the outliers are placed at random positions and (ii) the data is subsampled in a completely random way.
\begin{figure}[ht]
\begin{center}
\subfigure[atan$(80\%)$]{\includegraphics[width=0.22\textwidth]{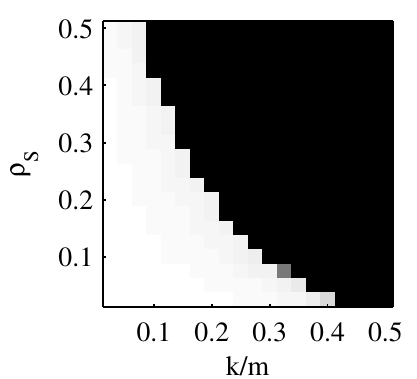}}
\subfigure[atan$(50\%)$]{\includegraphics[width=0.22\textwidth]{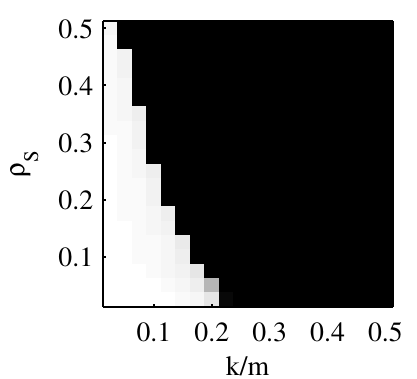}}
\subfigure[SpaRCS$(80\%)$]{\includegraphics[width=0.22\textwidth]{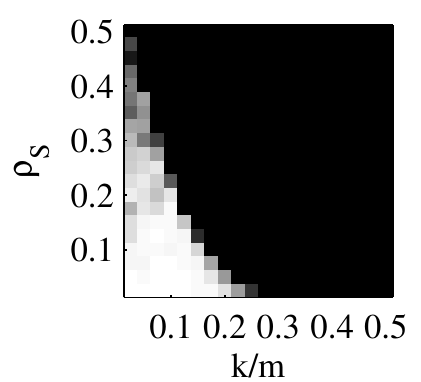}}
\subfigure[SpaRCS$(50\%)$]{\includegraphics[width=0.22\textwidth]{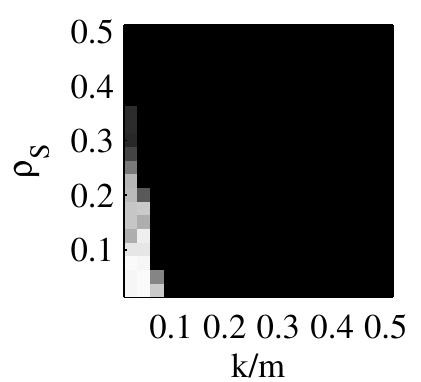}}
\caption{Phase transitions at recovering a low-rank matrix from incomplete observations}
\label{fig:incomplete}
\end{center}
\end{figure}

\subsection{Noisy reconstruction}
To investigate the behavior of the proposed algorithm in the presence of additive Gaussian noise we perform the following experiments: $\rho$ and $k/m$ are fixed to $0.1$ and a Gaussian noise term $N \in \mathbb{R}^{m \times n}$ of a particular energy level is added to the data. We run our alternating minimization method for $10$ iterations and compare the performance and the runtime against the above mentioned competing algorithms. Intuitively, a low-rank and sparse decomposition on a noisy dataset should mostly affect the sparse component, as Gaussian noise is full rank and only a small amount of noise will affect the low-rank component.
Figure \ref{fig:noisy} illustrates the recovery precision for the low-rank matrix, where the SNR measures the relation between the energy of $L$ and $N$. As the results reveal, the proposed rank-controlling method provides a rather noise-robust estimation of the subspace, which supports our choice of a simple model. Like all other methods, the reconstruction quality deteriorates with increasing noise level. However, at a moderate noise level our algorithm performs equally well as \emph{GoDec}, which in contrast to our method requires additional knowledge about the cardinality of the sparse component.

Concerning the average runtime for solving the noisy decomposition task on a desktop computer in MATLAB, \emph{GoDec} performs fastest in $1.5$ seconds and the \emph{IALM} method requires about $2.7$ seconds. Our framework needs an average runtime between $3$ (\emph{atan}) and $3.7$ (\emph{lpnorm}) seconds to solve the task and outperforms \emph{SpaRCS} and \emph{GRASTA}, which require $6.1$ and $11$ seconds, respectively.
Lastly, \emph{EALM} is rather costly with $30$ seconds runtime.
\begin{figure}[H]
\includegraphics[width=0.9\textwidth]{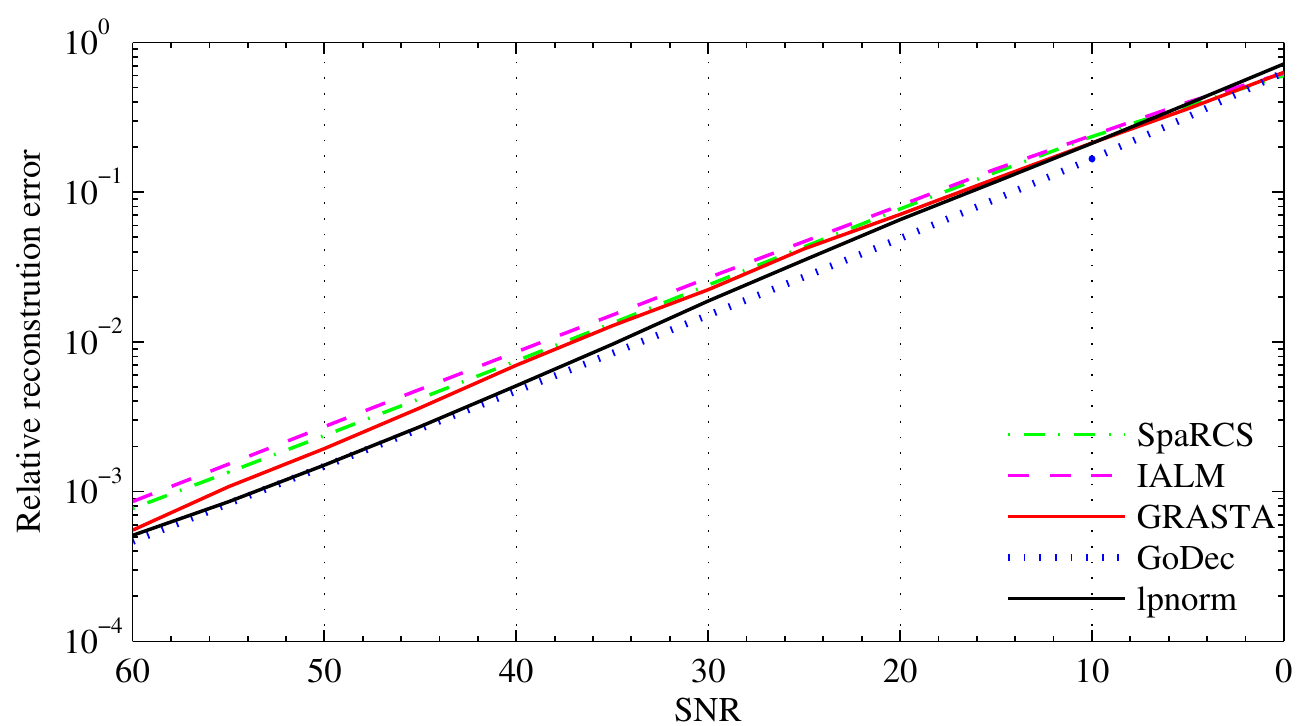}
\caption{Relative subspace recovery error at different Signal-to-Noise-Ratios}
\label{fig:noisy}
\end{figure}
\subsection{Subspace tracking on a real-world example}
\begin{figure}[ht]
\begin{center}
\subfigure[original $x$]{\includegraphics[width=0.18\textwidth]{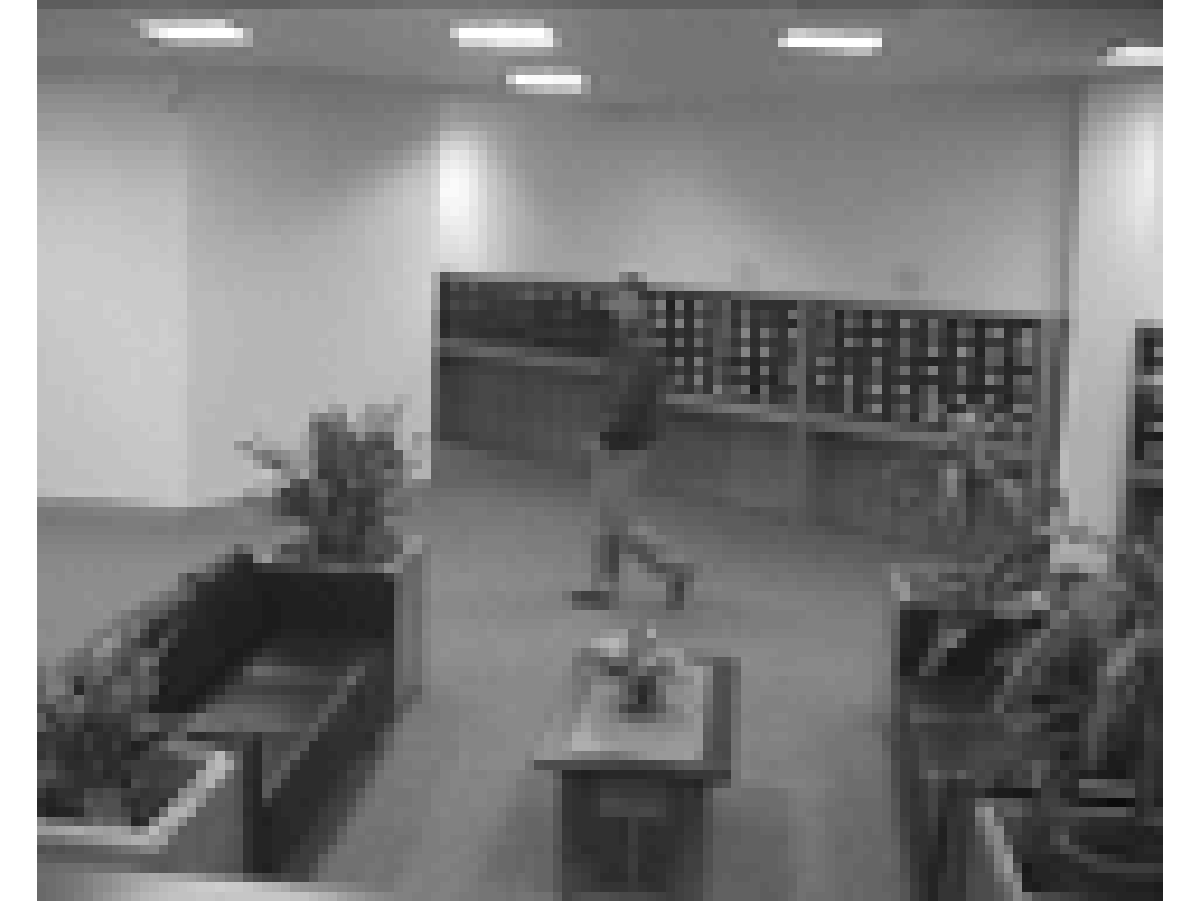}}
\subfigure[$\hat{l}\; (\mu=2)$]{\includegraphics[width=0.18\textwidth]{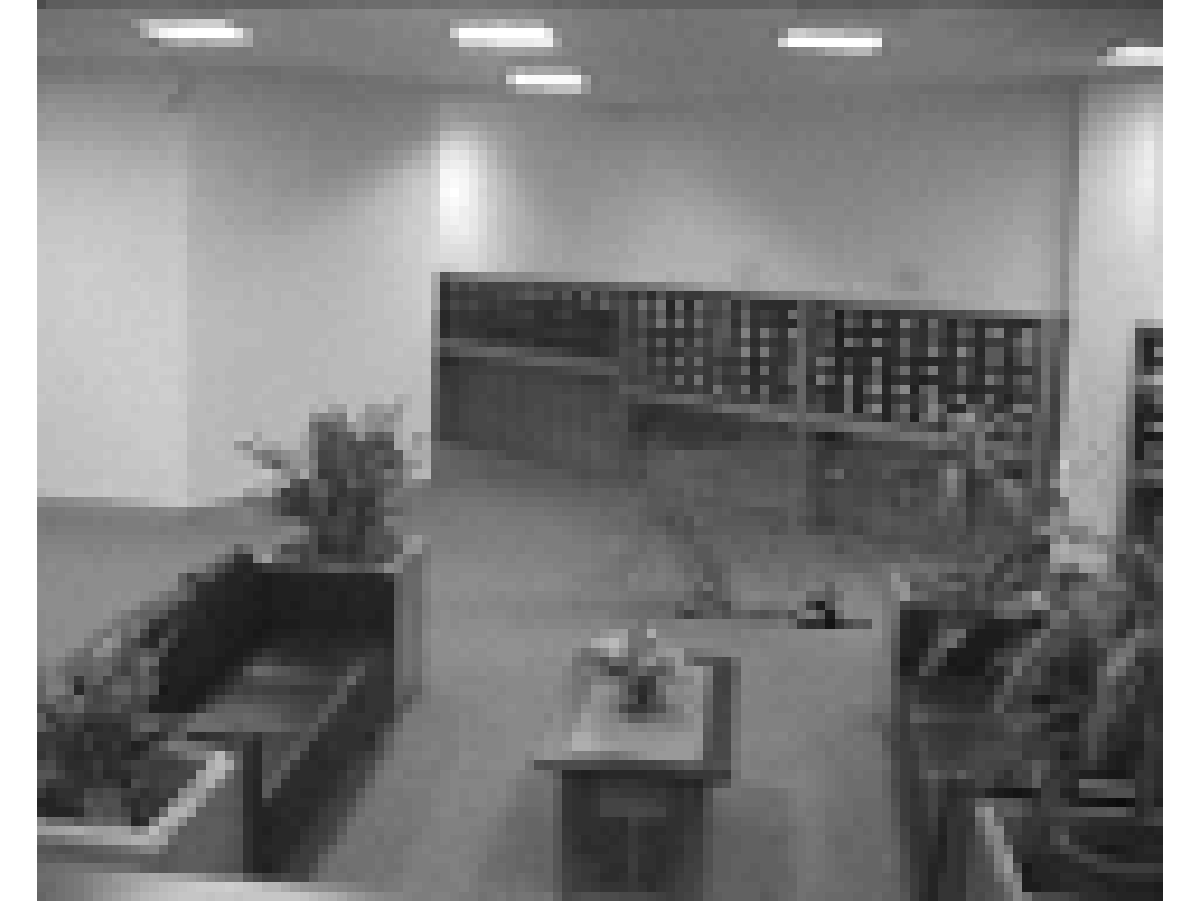}}
\subfigure[$\hat{s}\; (\mu=2)$]{\includegraphics[width=0.18\textwidth]
{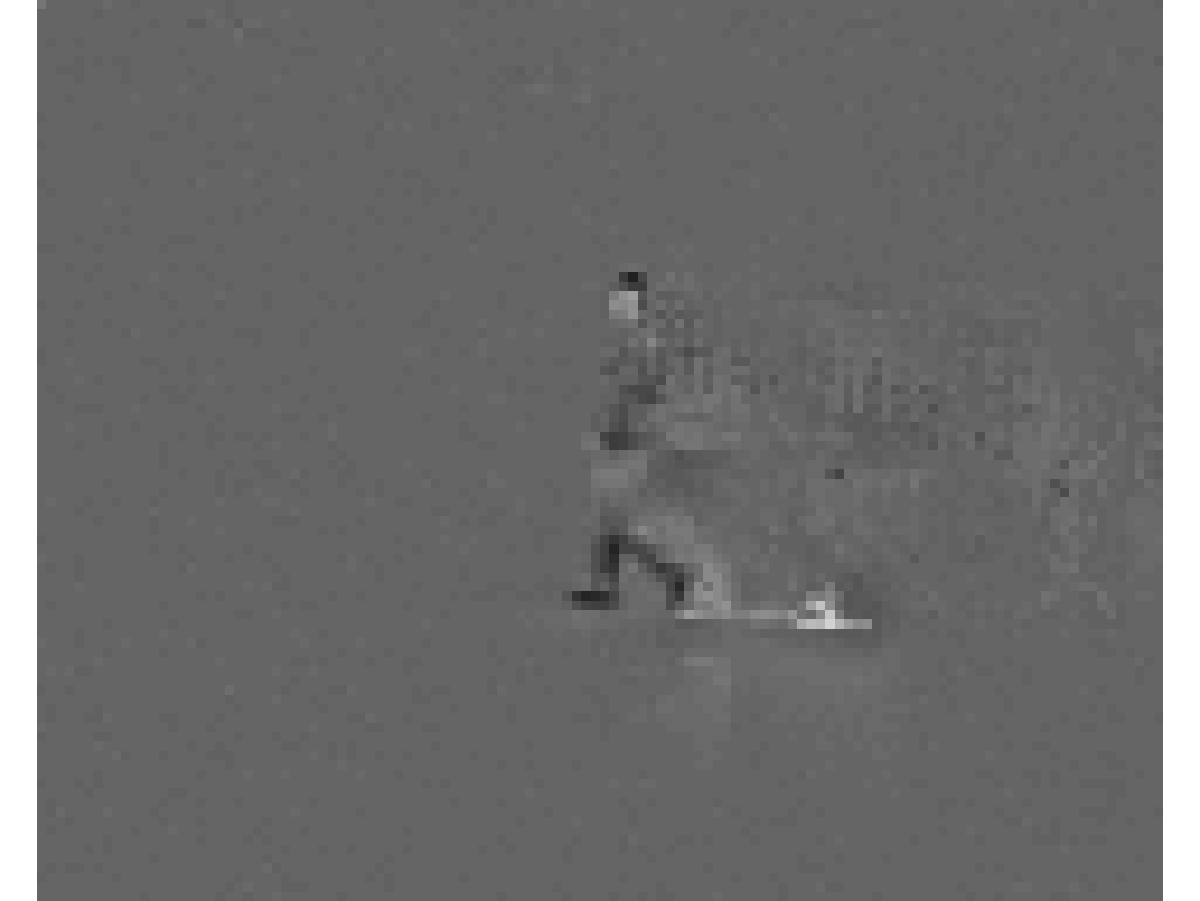}}
\subfigure[$\hat{l}\; (\mu=0.01)$]{\includegraphics[width=0.18\textwidth]{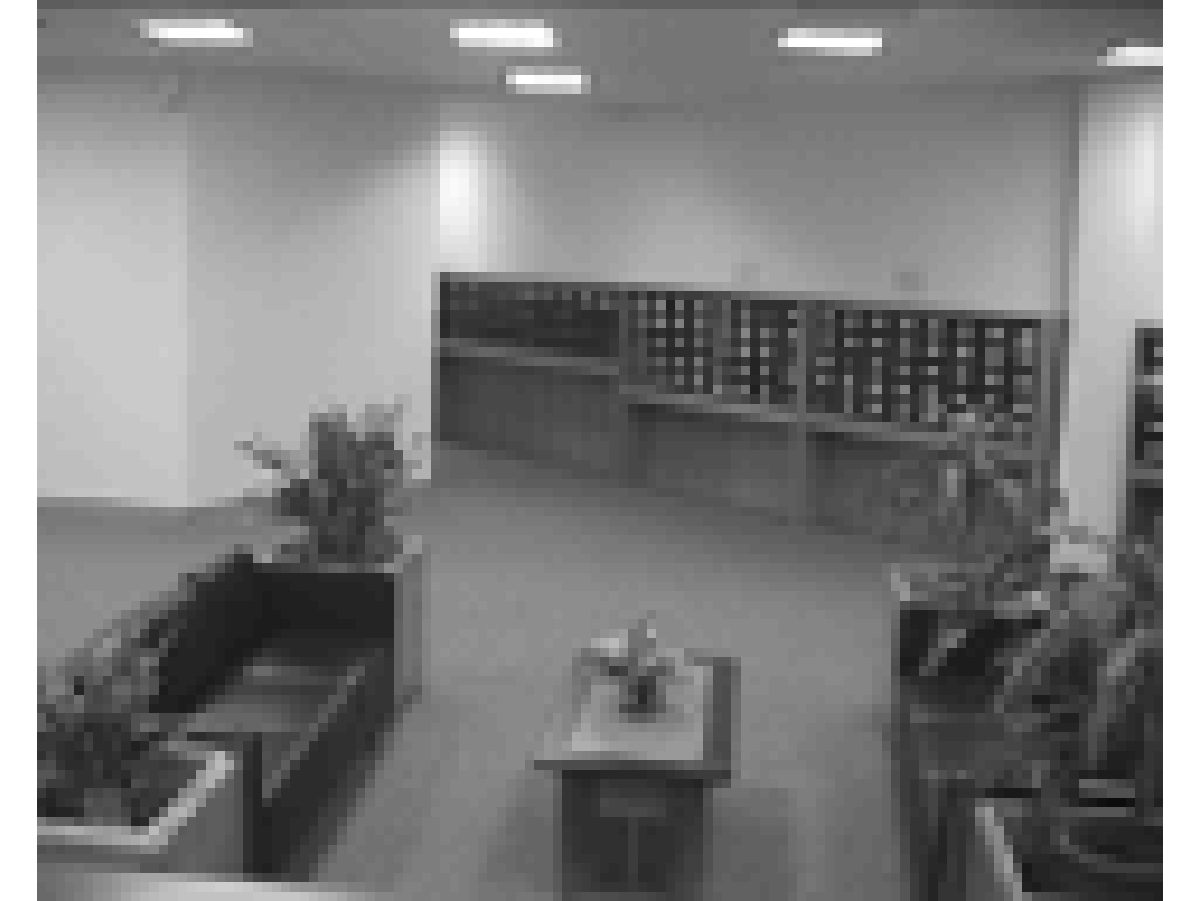}}
\subfigure[$\hat{s}\; (\mu=0.01)$]{\includegraphics[width=0.18\textwidth]
{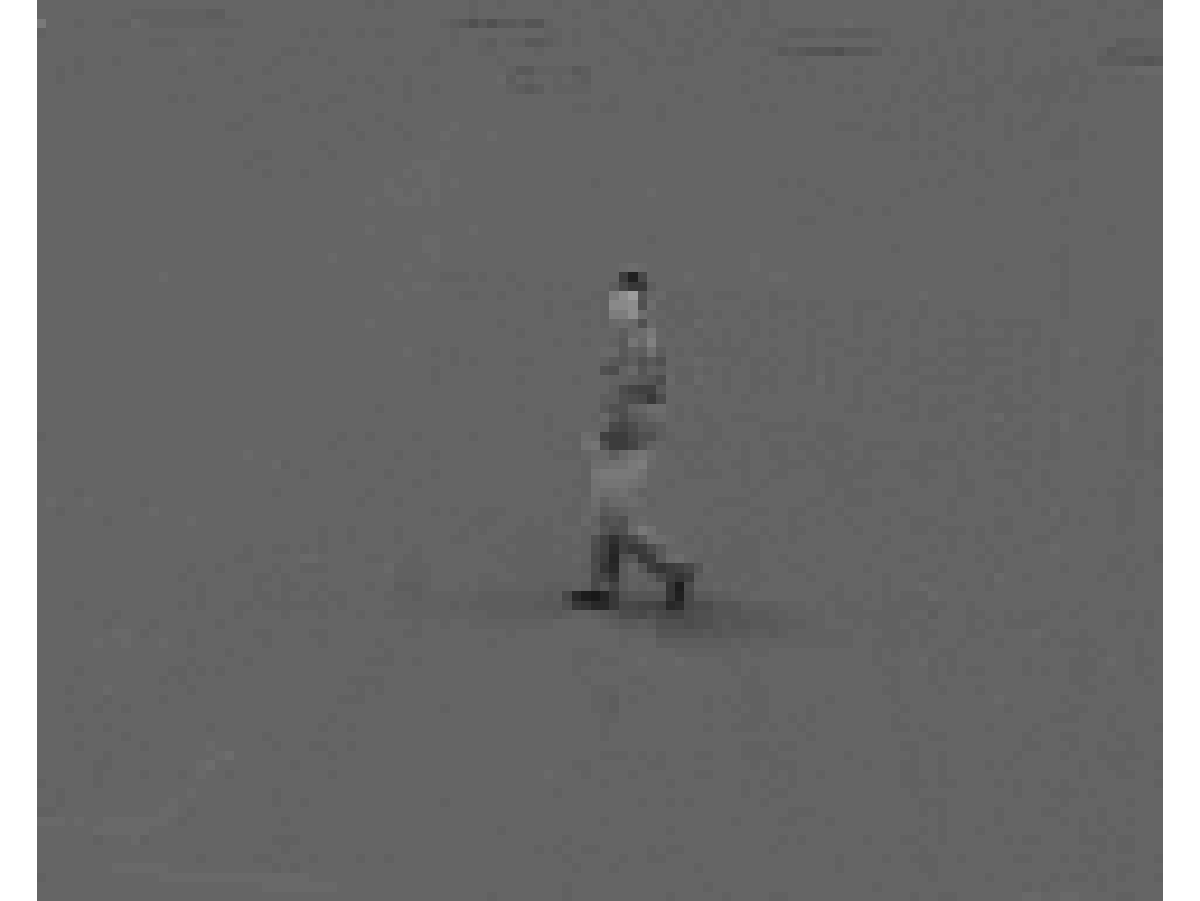}}
\caption{Subspace tracking - Background and foreground for different smoothing parameters}
\label{fig:tracking_LS}
\end{center}
\end{figure}
For the purpose of robustly tracking the underlying subspace, we selected the dataset \emph{lobby} from \cite{Li2004} and perform the widely popular background subtraction task. In this task, the background of a pixel-wisely sampled video is modeled by a low-rank approximation and foreground objects can automatically be extracted as they are assumed to be sparse. The selected scene shows an office scenario with people occasionally walking through the scene. It is especially challenging as the lighting conditions change significantly after about $400$ frames. We initialize our tracking algorithm with the first $50$ frames of the sequence using the proposed alternating minimization method for $10$ steps and perform our tracking algorithm on the subsequent frames. From our $\ell_0$-surrogates we choose \eqref{eq:atansquare} and upper-bound the desired rank $k$ by $2$. To demonstrate the influence of the smoothing parameter we fix $\mu=2$ for one experiment and for the other we shrink $\mu$ from $2$ to $0.01$ in the initialization phase and leave it for the following tracking procedure. The choice of a suitable weighting parameter is a trade-off between reaction time (i.e.~how long it takes to adapt the subspace) and the risk of overfitting the subspace (i.e.~$\hat{l} \approx x$), which leads to ghost images in the reconstruction. For all experiments we select $w=0.05$ as a weighting factor.

The sample observation of frame $\#\;365$ in Figure \ref{fig:tracking_LS} illustrates that in both configurations the subspace $\hat{l}$ is successfully recovered, as the sparse foreground estimates $\hat{s} = x - \hat{l}$ (\ref{fig:tracking_LS}(c) and (e)) contain only the moving person in the scene. However, the smaller $\mu$ is selected, the less blurry the extracted silhouette, as becomes clear from the comparison between the two experiments. The more sparsity on the residual is enforced, the more effectively can ghost images from previous observations be suppressed in the extracted foreground objects.
\begin{figure}[H]
\begin{center}
\subfigure[frame $\#425$]{\includegraphics[width=0.18\textwidth]{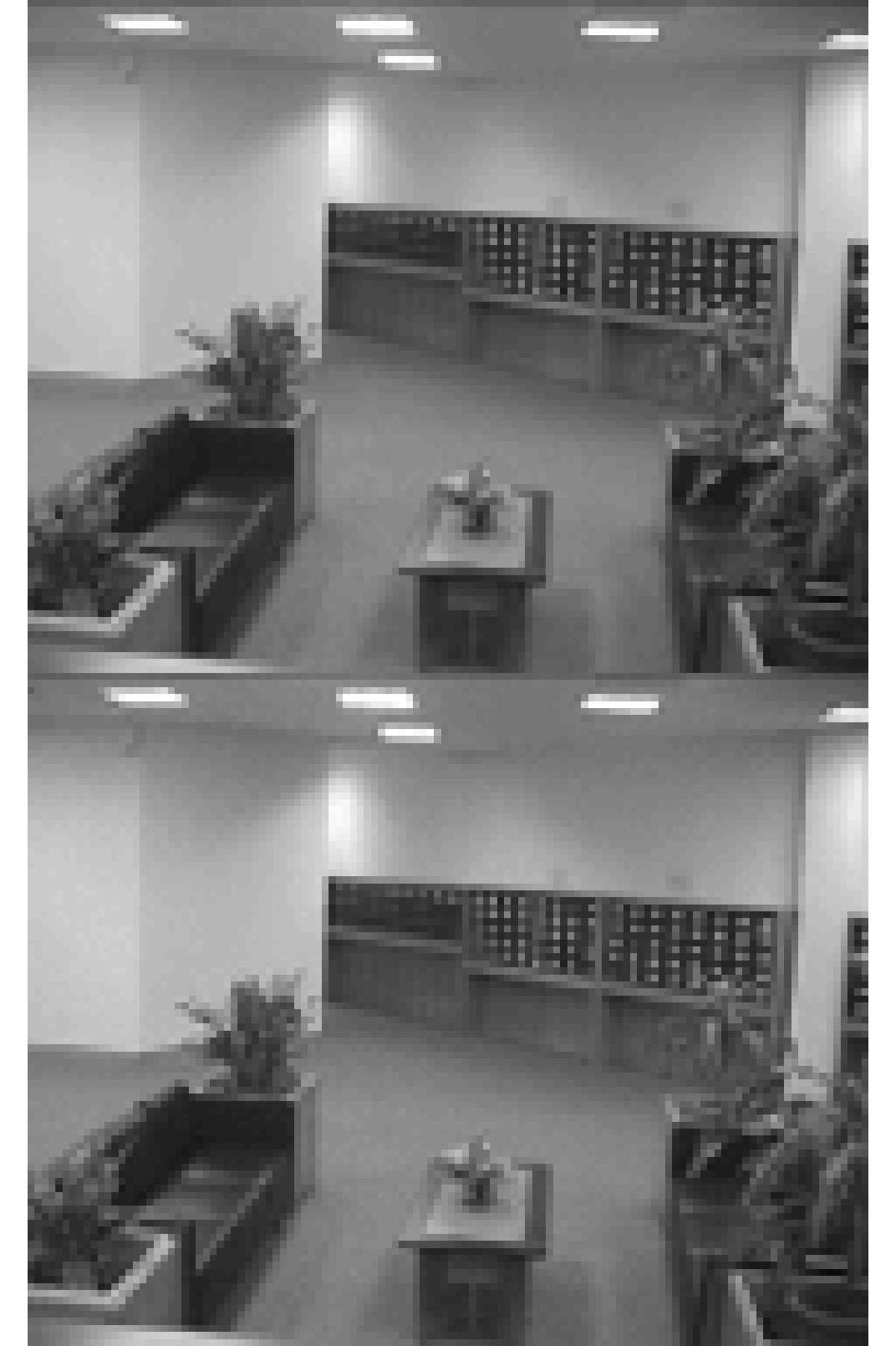}}
\subfigure[frame $\#437$]{\includegraphics[width=0.18\textwidth]{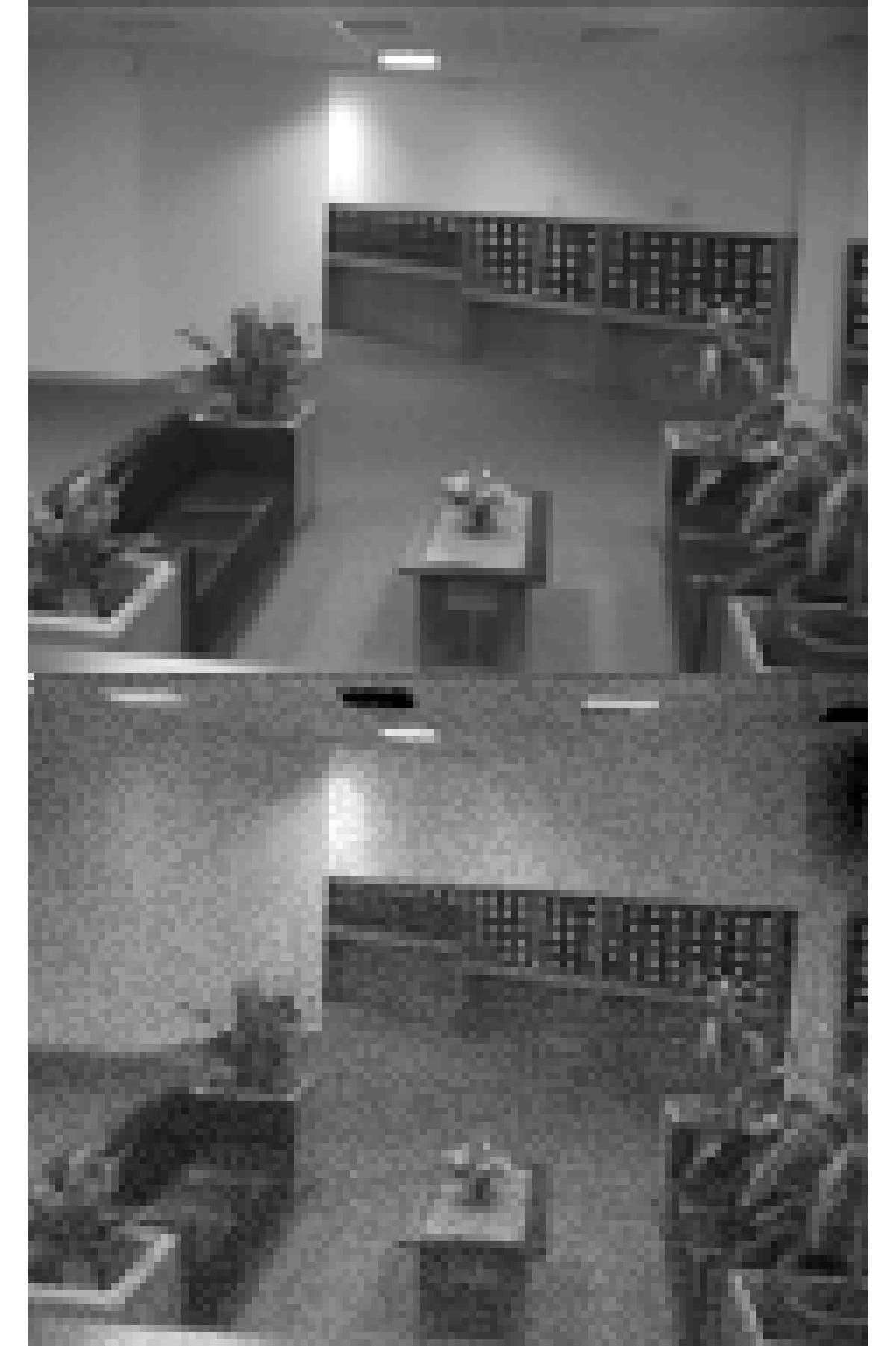}}
\subfigure[frame $\#450$]{\includegraphics[width=0.18\textwidth]{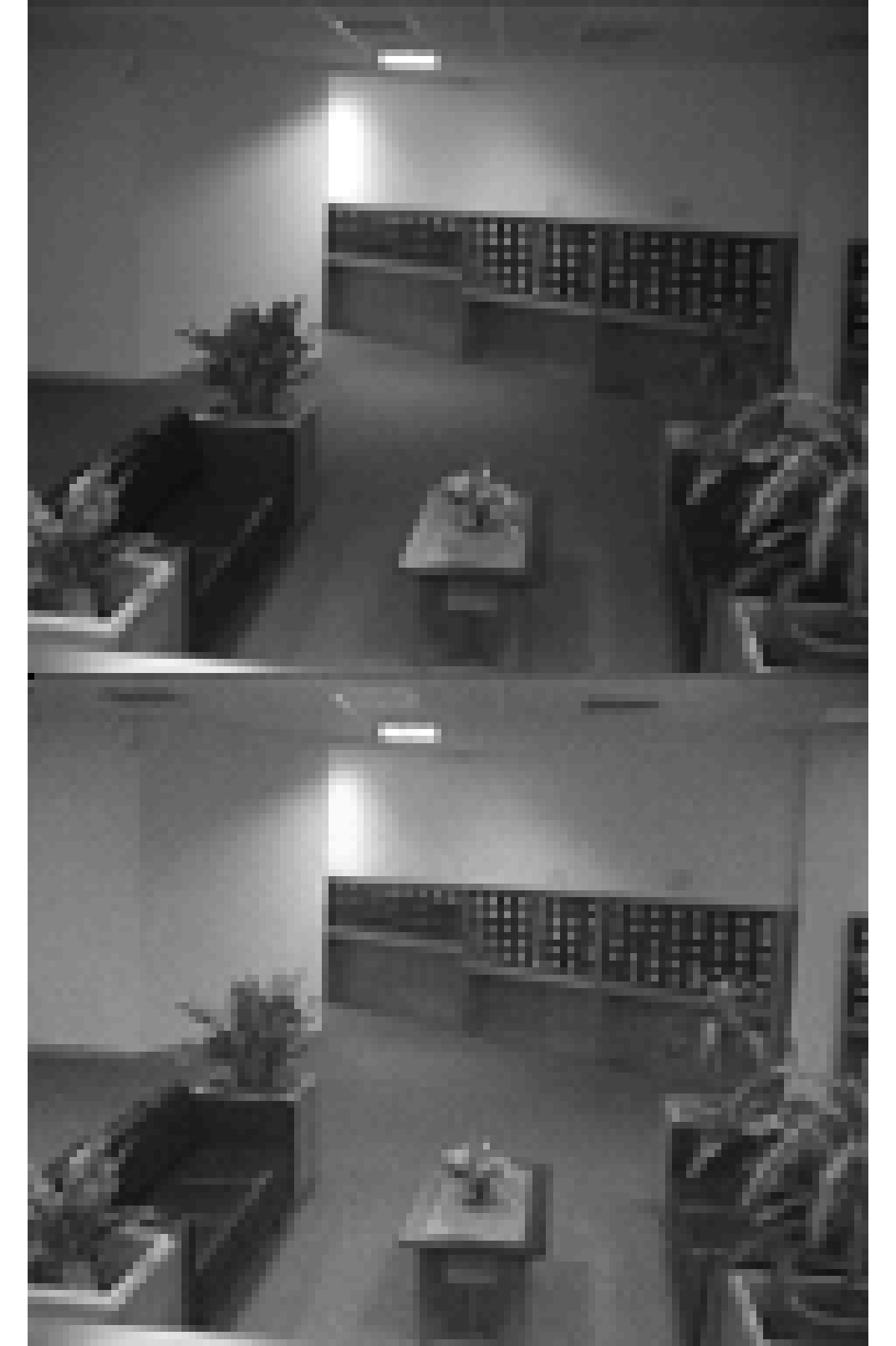}}
\subfigure[frame $\#462$]{\includegraphics[width=0.18\textwidth]{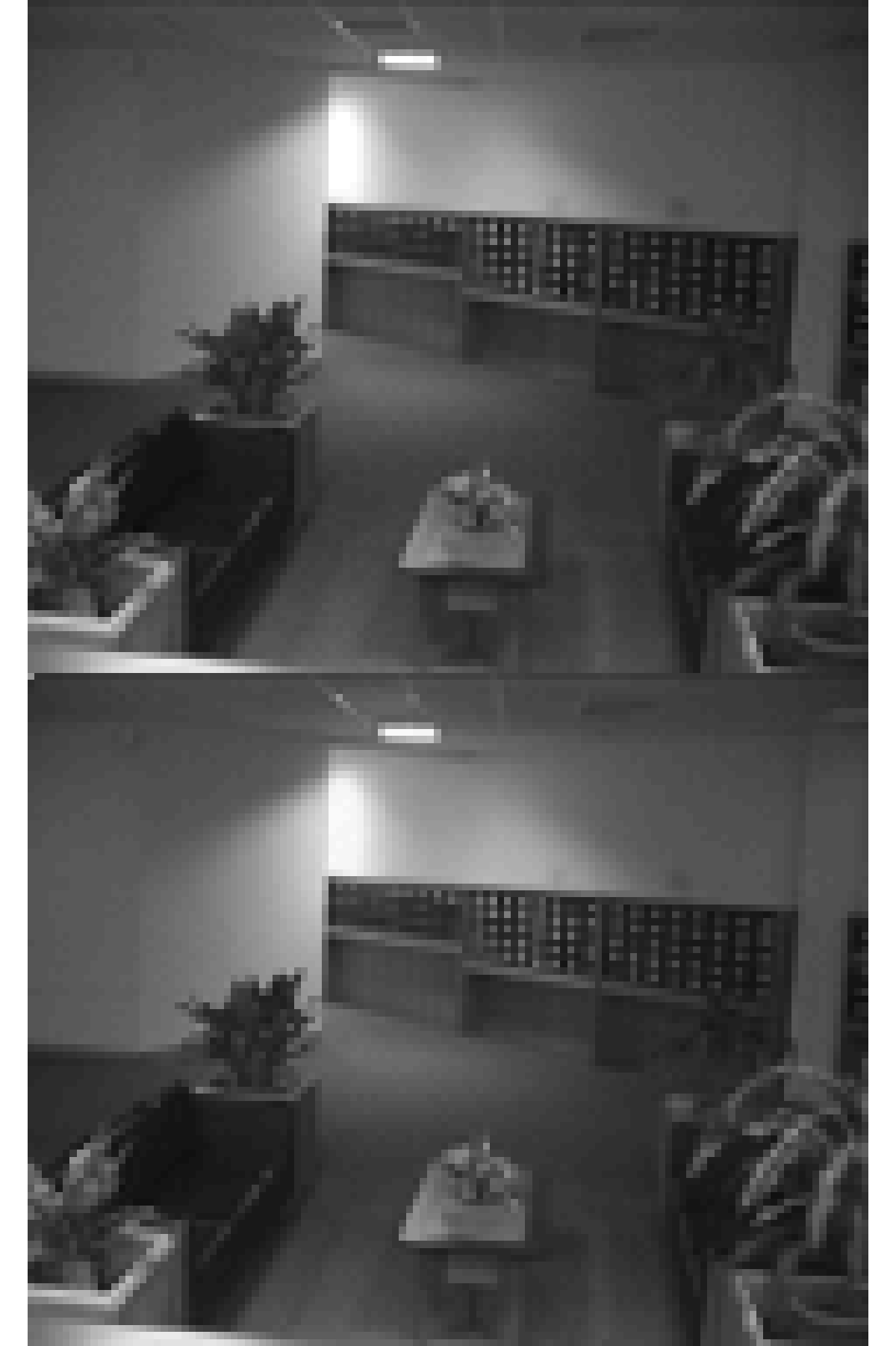}}
\subfigure[frame $\#475$]{\includegraphics[width=0.18\textwidth]{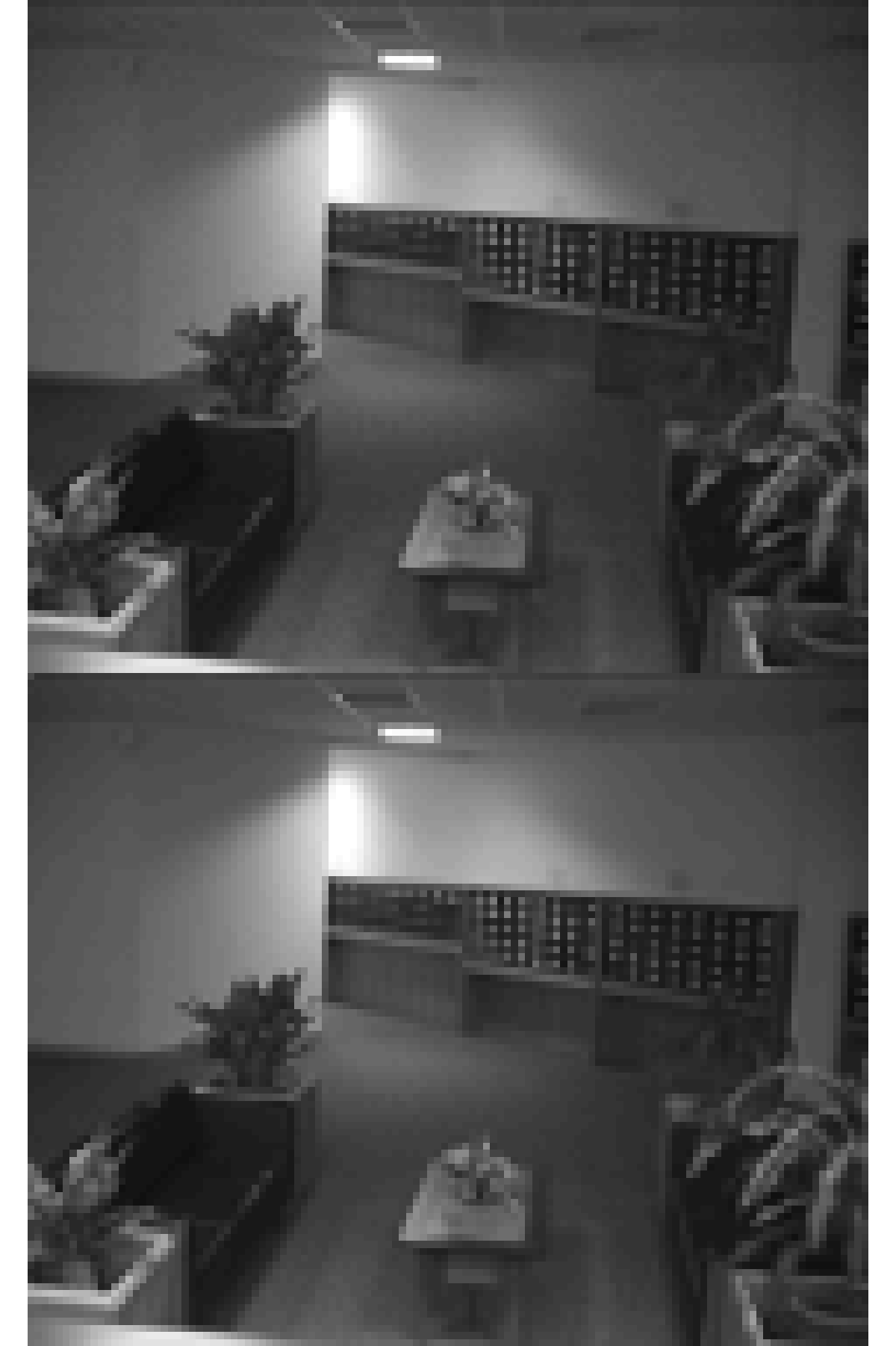}}
\caption{Tracking a change in room illumination. Upper: original video, lower: low-rank approximation ($\mu=2$)}
\label{fig:tracking_switch}
\end{center}
\end{figure}

Figure \ref{fig:tracking_switch} illustrates the behavior of the algorithm when a change occurs in the background. Although the lighting conditions following frame $\#\;425$ have not been observed during the initialization, our tracking algorithm adjusts the subspace to the new background within a small number of frames.

\section{Conclusion}
We present a framework for Robust PCA that is able to recover a low-rank matrix from a data set corrupted by sparse outliers and missing data. Experiments show that our method covers a wider range of scenarios in terms of higher rank and greater number of outliers than other state of the art methods. Even if the data is incompletely observed a comparably good performance is achieved. The same holds true for data sets corrupted by additive Gaussian noise.

Instead of tackling a convex relaxation of the problem, our algorithm reaches as closely as possible to the ideal Robust PCA objective by modelling the problem as simply as possible and making no additional assumptions, e.g.~on the number of sparse outliers. Instead of the widely-popular $\ell_1$-measure we propose using $\ell_0$-surrogates, which can be adjusted so that the optimization is performed rather smoothly or, in the limit, sparsity is enforced in an $\ell_0$-like behavior, which leads to superior performance. In order to obtain an inherent upper bound on the dimension of the dominant subspace we propose an optimization problem on the Grassmannian and derive an efficient retraction that saves time and storage and makes the algorithm applicable for large-scale problems. Furthermore we show how the method can be adapted to allow tracking a subspace that varies over time and use real-world data to demonstrate the performance.

\bibliographystyle{spbasic}      
\bibliography{bibliography}   

\begin{thebibliography}{24}
\providecommand{\natexlab}[1]{#1}
\providecommand{\url}[1]{{#1}}
\providecommand{\urlprefix}{URL }
\expandafter\ifx\csname urlstyle\endcsname\relax
  \providecommand{\doi}[1]{DOI~\discretionary{}{}{}#1}\else
  \providecommand{\doi}{DOI~\discretionary{}{}{}\begingroup
  \urlstyle{rm}\Url}\fi
\providecommand{\eprint}[2][]{\url{#2}}

\bibitem[{Absil et~al(2008)Absil, Mahony, and Sepulchre}]{Absil2008}
Absil PA, Mahony R, Sepulchre R (2008) Optimization Algorithms on Matrix
  Manifolds. Princeton University Press, Princeton, NJ

\bibitem[{Balzano et~al(2010)Balzano, Nowak, and Recht}]{Balzano2010}
Balzano L, Nowak R, Recht B (2010) Online identification and tracking of
  subspaces from highly incomplete information. In: Communication, Control, and
  Computing, 2010 48th Annual Allerton Conference on, pp 704--711

\bibitem[{Boumal and Absil(2011)}]{Boumal:11}
Boumal N, Absil PA (2011) {RTRMC: A Riemannian trust-region method for low-rank
  matrix completion}. In: Advances in Neural Information Processing Systems, pp
  406--414

\bibitem[{Cai et~al(2010)Cai, Cand\`{e}s, and Shen}]{Cai2008}
Cai J, Cand\`{e}s EJ, Shen Z (2010) A singular value thresholding algorithm for
  matrix completion. SIAM J on Optimization 20:1956--1982

\bibitem[{Cand\`{e}s et~al(2011)Cand\`{e}s, Li, Ma, and Wright}]{Candes2011}
Cand\`{e}s E, Li X, Ma Y, Wright J (2011) {Robust principal component
  analysis?} Journal of ACM 58(3):1--37

\bibitem[{Chartrand and Staneva(2008)}]{Chartrand2008}
Chartrand R, Staneva V (2008) Restricted isometry properties and nonconvex
  compressive sensing. Inverse Problems 24(3):1--14

\bibitem[{Chen et~al(2011)Chen, Xu, Caramanis, and Sanghavi}]{Chen2011}
Chen Y, Xu H, Caramanis C, Sanghavi S (2011) {Robust Matrix Completion and
  Corrupted Columns}. In: International Conference on Machine Learning, vol~2,
  pp 873--880

\bibitem[{Ding et~al(2006)Ding, Zhou, He, and Zha}]{Ding2006}
Ding C, Zhou D, He X, Zha H (2006) {R1-{PCA}: {R}otational invariant L1-norm
  {P}rincipal {C}omponent {A}nalysis for robust subspace factorization}. In:
  23rd international conference on Machine learning, ACM, New York, NY, USA, pp
  281--288

\bibitem[{Gasso et~al(2009)Gasso, Rakotomamonjy, and Canu}]{Gasso2009}
Gasso G, Rakotomamonjy A, Canu S (2009) {Recovering Sparse Signals With a
  Certain Family of Nonconvex Penalties and DC Programming}. IEEE Transactions
  on Signal Processing 57(12):4686 --4698

\bibitem[{Golub and {{Van Loan}}(1996)}]{Golub96}
Golub H, {{Van Loan}} CF (1996) Matrix computations. Johns Hopkins University
  Press, Baltimore, MD, USA

\bibitem[{He et~al(2012)He, Balzano, and Szlam}]{He2012}
He J, Balzano L, Szlam A (2012) {Incremental gradient on the Grassmannian for
  online foreground and background separation in subsampled video}. In: IEEE
  Conference on Computer Vision and Pattern Recognition, pp 1568--1575

\bibitem[{Helmke et~al(2007)Helmke, H{\"u}per, and Trumpf}]{Helmke07}
Helmke U, H{\"u}per K, Trumpf J (2007) {Newton's method on Grassmann
  manifolds}. ArXiv e-prints \eprint{0709.2205}

\bibitem[{Keshavan and Montanari(2010)}]{Keshavan2010}
Keshavan RH, Montanari A (2010) {Matrix completion from noisy entries}. The
  Journal of Machine Learning Research 11:2057--2078

\bibitem[{Kleinsteuber and H\"{u}per(2007)}]{Kleinsteuber2007}
Kleinsteuber M, H\"{u}per K (2007) {An Intrinsic CG Algorithm for Computing
  Dominant Subspaces}. In: IEEE International Conference on Acoustics, Speech
  and Signal Processing, 7, pp 1405--1408

\bibitem[{Kwak(2008)}]{Kwak2008}
Kwak N (2008) {Principal component analysis based on l1-norm maximization.}
  IEEE Transactions on Pattern Analysis and Machine Intelligence 30(9):1672--80

\bibitem[{Li et~al(2004)Li, Huang, Gu, and Tian}]{Li2004}
Li L, Huang W, Gu IYH, Tian Q (2004) Statistical modeling of complex
  backgrounds for foreground object detection. Image Processing, IEEE
  Transactions on 13(11):1459 --1472

\bibitem[{Lin et~al(2010)Lin, Chen, Wu, and Ma}]{Lin2010}
Lin Z, Chen M, Wu L, Ma Y (2010) {The augmented lagrange multiplier method for
  exact recovery of corrupted low-rank matrices}. Arxiv preprint arXiv:10095055
  \eprint{arXiv:1009.5055v2}

\bibitem[{Meyer et~al(2011)Meyer, Bonnabel, and Sepulchre}]{Meyer2011}
Meyer G, Bonnabel S, Sepulchre R (2011) {Linear regression under fixed-rank
  constraints: A Riemannian approach}. In: International Conference on Machine
  Learning, pp 545--552

\bibitem[{Ring and Wirth(2012)}]{Ring2012}
Ring W, Wirth B (2012) {Optimization methods on Riemannian manifolds and their
  application to shape space}. SIAM Journal on Optimization 22(2):596--627

\bibitem[{Shalev-Shwartz et~al(2011)Shalev-Shwartz, Gonen, and
  Shamir}]{Shalev-Shwartz2011}
Shalev-Shwartz S, Gonen A, Shamir O (2011) {Large-Scale Convex Minimization
  with a Low-Rank Constraint}. In: International Conference on Machine
  Learning, pp 329--336

\bibitem[{Shalit et~al(2010)Shalit, Weinshall, and Chechik}]{Shalit2010}
Shalit U, Weinshall D, Chechik G (2010) {Online learning in the manifold of
  low-rank matrices}. Advances in Neural Information Processing Systems
  23:2128--2136

\bibitem[{Waters et~al(2011)Waters, Sankaranarayanan, and
  Baraniuk}]{Waters2011}
Waters A, Sankaranarayanan AC, Baraniuk RG (2011) {SpaRCS: Recovering Low-Rank
  and Sparse Matrices from Compressive Measurements}. In: Advances in Neural
  Information Processing Systems

\bibitem[{Wright et~al(2009)Wright, Ganesh, Rao, Peng, and Ma}]{wright:09}
Wright J, Ganesh A, Rao S, Peng Y, Ma Y (2009) {Robust Principal Component
  Analysis: Exact Recovery of Corrupted Low-Rank Matrices via Convex
  Optimization}. In: Advances in Neural Information Processing Systems, pp
  2080--2088

\bibitem[{Zhou and Tao(2011)}]{Zhou2011}
Zhou T, Tao D (2011) {GoDec: Randomized low-rank \& sparse matrix decomposition
  in noisy case}. In: International Conference on Machine Learning, pp 33--40

\end{thebibliography}

%
%

\end{document}